\RequirePackage{amsthm}

\documentclass[sn-mathphys-ay]{sn-jnl}% Math and Physical Sciences Numbered Reference Style 
\usepackage{placeins}
\usepackage{graphicx}%
\usepackage{multirow}%
\usepackage{amsmath,amssymb,amsfonts}%
\usepackage{mathrsfs}%
\usepackage[title]{appendix}%
\usepackage{xcolor}%
\usepackage{textcomp}%
\usepackage{manyfoot}%
\usepackage{booktabs}%
\usepackage{listings}%
\usepackage{enumitem}

\usepackage{lmodern}
\usepackage{nicefrac}       % compact symbols for 1/2, etc.
\usepackage{microtype}      % microtypography
\usepackage{lipsum}
\usepackage{bbm}
\usepackage{physics}
\usepackage{bm}
\usepackage[short]{optidef}
\newcommand{\ra}[1]{\renewcommand{\arraystretch}{#1}}
\newcommand{\RomanNumeralCaps}[1]
    {\MakeUppercase{\romannumeral #1}}
\DeclarePairedDelimiter\floor{\lfloor}{\rfloor}
\usepackage{subcaption}
\usepackage{tikz}
\usepackage{cases}
\usepackage{xfrac}
\usepackage{acronym}

\usepackage{algorithmic}
\usepackage[ruled,vlined]{algorithm2e}
\SetKwInput{KwInput}{Input}                % Set the Input
\SetKwInput{KwOutput}{Output} 
\usepackage{mathtools}

\theoremstyle{thmstyleone}%
\newtheorem{theorem}{Theorem}%  meant for continuous numbers
\newtheorem{proposition}[theorem]{Proposition}%
\newtheorem{lemma}{Lemma}%

\theoremstyle{thmstylethree}%
\newtheorem{remark}{Remark}%

\newcommand{\jnm}[1]{\textcolor{black}{#1}}

\begin{document}

\title[Optimal Control of Fluid Restless Multi-armed
Bandits: A Machine Learning Approach]{Optimal Control of Fluid Restless Multi-armed
Bandits: A Machine Learning Approach}


\author*[1]{\fnm{Dimitris} \sur{Bertsimas}}\email{dbertsim@mit.edu}
\equalcont{These authors contributed equally to this work.}

\author[2]{\fnm{Cheol Woo} \sur{Kim}}\email{acwkim@mit.edu}
\equalcont{These authors contributed equally to this work.}

\author[3]{\fnm{Jos\'e} \sur{Ni\~no-Mora}}\email{jose.nino@uc3m.es}
\equalcont{These authors contributed equally to this work.}

\affil*[1]{\orgdiv{Sloan School of Management}, \orgname{Massachusetts Institute of Technology}, \orgaddress{\city{Cambridge}, \postcode{02142}, \state{MA}, \country{USA}}}

\affil[2]{\orgdiv{Operations Research Center}, \orgname{Massachusetts Institute of Technology}, \orgaddress{\city{Cambridge}, \postcode{02142}, \state{MA}, \country{USA}}}

\affil[3]{\orgdiv{Departamento de Estad\'{\i}stica}, \orgname{Universidad Carlos \RomanNumeralCaps{3} de Madrid}, \orgaddress{\city{Getafe}, \postcode{28903}, \state{Madrid}, \country{Spain}}}


\abstract{
We present a novel machine learning framework for the optimal control of fluid restless multi-armed bandit problems (FRMABPs) with state equations that are either affine or quadratic in the state variables. By establishing fundamental properties of FRMABPs, we develop an efficient numerical algorithm that generates a comprehensive training set by solving multiple instances with diverse initial states. We further enhance this training set by applying a nonlinear transformation to the feature vectors, leveraging structural properties of FRMABPs. A time-dependent state feedback policy is then learned using Optimal Classification Trees with hyperplane splits (OCT-H). We test our approach on machine maintenance, epidemic control, and fisheries control problems, demonstrating that our method yields high-quality state feedback policies. Furthermore, once a policy is learned, it achieves a speed-up of up to 26 million times compared to the direct numerical algorithm.
\vspace{1.5em} \\
\noindent{\textbf{Publication note:} Accepted for publication in \emph{Machine Learning}, 115, 70 (2026). Final version available at \url{https://doi.org/10.1007/s10994-026-07022-0}.}}

\keywords{Restless Multi-Armed Bandit, Fluid Approximation, Optimal Control, Decision Trees}



\maketitle

\section{Introduction}
\label{sec:intro}
We  \jnm{consider a} continuous-time, continuous-state deterministic formulation of \jnm{the \emph{restless multi-armed bandit problem} (RMABP)}, which we refer to as the \jnm{fluid RMABP} (FRMABP). \jnm{RMABPs were} introduced \jnm{in} \citep{whit88} \jnm{as discrete-state stochastic control problems, typically formulated as \emph{Markov decision processes} (MDPs).  These models address optimal sequential resource allocation among multiple projects, where each project may operate in an active or a passive mode at any time, subject to a constraint on the number of projects that can be active concurrently. The state of each project evolves stochastically under different dynamics and yields different rewards depending on the chosen mode, and the goal is to determine a control policy that optimizes the discounted or long-run average expected reward over a time horizon.} The model has \jnm{generated a vast literature due to its } numerous real-world applications, which include healthcare \citep{health}, machine maintenance \citep{abbuMakis19}, and wireless communication \citep{wireless}, among others. For a recent review of the field, see \citep{nmmath23}.

\jnm{Yet, solving RMABPs} to optimality is PSPACE-hard \citep{paptsi99}.
Consequently, the standard approach is to rely on suboptimal heuristic policies, among which the \emph{Whittle index policy} \citep{whit88} is the most widely applied. However, this heuristic is applicable only to the limited class of \emph{indexable} models - those for which a Whittle index exists. Furthermore, although these heuristics are proven to be asymptotically optimal under certain conditions \citep{wewe90,gastetal24}, the required conditions are often notoriously hard to verify.

To mitigate the complexity of solving discrete-state stochastic control problems, it is common to study their \emph{fluid} (continuous-state, deterministic) relaxations, which are typically more tractable analytically. This paradigm is classical in the analysis and control of \emph{multiclass queueing networks} (MCQNs): the stability of a stochastic MCQN can be inferred from the stability of its fluid model \citep{dai1995,Stolyar1995}. Moreover, fluid MCQNs provide effective surrogates for policy design in the original stochastic systems, yielding controllers with strong empirical performance and, in many settings, provable asymptotic optimality \citep{Mag1999,Mag2000,robustfluid,novetal22}. For background, see \citet[Chs.~7 and~16]{queuebook}.

% On the other hand, continuous-time deterministic optimal control problems have historically been proven to be useful to model numerous real-world problems in engineering \citep{optbook2}, natural and social sciences \citep{optbook1, grassetal08}  and operations research \citep{Sethi2021}.  

\textcolor{black}{In the RMABP literature, fluid models have been studied extensively, but from a viewpoint distinct from the fluid models used for MCQNs. Specifically, one considers mean-field (fluid) scaling of a discrete-state RMABP in which the total number of projects and the number of active projects grow to infinity in a fixed ratio. The resulting deterministic limit is expressed in terms of state-action occupancy measures ---the proportion of projects in a given state receiving a given action--- whose dynamics are governed by ODEs (for continuous-time systems) or discrete-time maps (for discrete-time systems). These fluid models are then used both to synthesize policies for the original stochastic system ---by lifting optimal fluid controls—and to establish asymptotic optimality of given policies in the large-system limit.
This approach originates with \cite{wewe90}, and has been further developed in, e.g., \cite{Verloop2016,brownSmith20, zhangFrazier21, brownZhang23, hongetal23, gastetal24, yan24, yanetal24}.}

 \textcolor{black}{The fluid models considered here are of the MCQN type: each project’s state evolves as a deterministic, real-valued system governed by ODEs. This viewpoint has received limited attention in the RMABP literature. A notable exception is \cite{larn16}, which analyzes a fluid version of Whittle’s relaxation for a specific RMABP and derives a corresponding fluid index policy that performs near-optimally in the reported experiments. However, \cite{larn16} focuses on a particular model and relies on relaxing the original coupling (resource) constraint. In the general case ---where the fluid problems become continuous-time, continuous-state optimal control problems--- obtaining exact solutions without such assumptions or relaxations remains computationally challenging. This difficulty is especially acute in practice, where one may repeatedly have to solve optimal control problems from varying initial conditions.}

Recently, many learning-based approaches have been proposed to overcome the complexity of solving optimal control problems. \citep{LEE2021109421, pmlr-v120-kim20b, lutter23} develop reinforcement learning methods, while \citep{bertsimas2023optimalcontrolmulticlassfluid} propose using a decision tree algorithm, Optimal Classification Trees with Hyperplane Splits (OCT-H) \citep{OCT}, to solve fluid \jnm{MCQN} control problems. {\color{black} See also \citep{donge24a, donge24b}, where data-efficient machine learning approaches are proposed for controlling complex systems.}

In this work, {\color{black}to alleviate the computational burden of solving FRMABPs, we propose a machine learning approach based on OCT-H.} Classification tree algorithms are particularly appealing for learning optimal policies in continuous-time optimal control problems. These problems often have piecewise constant optimal policies, which decision trees with hyperplane splits can effectively learn \citep[Chapters 4 \& 8 in][]{Meyn2007, bertsimas2023optimalcontrolmulticlassfluid}. {\color{black} Once trained, the learned policy can be applied to any new state at minimal computational cost, eliminating the need to run computationally expensive numerical algorithms.}

Our approach is similar to the class of methods known as \textit{imitation learning} in the reinforcement learning literature. 
\jnm{Given an optimal control model with fixed parameter values, we generate multiple instances which differ in their initial states only} and solve them using numerical algorithms to generate training data. Then, we use \jnm{a} supervised learning algorithm, OCT-H, to imitate the mapping from state and time to control observed in the generated trajectories.

\subsubsection*{\textbf{Notational Conventions}} Throughout this paper, we use boldface letters to denote vectors and matrices. The $i_{th} $ entry of a vector $\bm{x}$ is denoted $x_i$. We use $\bm{0}$ to denote the vector of zeros. We use $x(\cdot)$ to denote a real-valued function, and $\bm{x}(\cdot)$ to denote a vector whose entries are real-valued functions. We use $\bm{x}$ instead of $\bm{x}(\cdot)$ when it is clear from the context that $\bm{x}$ is referring to a vector of functions. 
\jnm{We also write $[n]$ to denote the set $\{1, \ldots, n\}$ for a positive integer $n$.}

\subsection{Problem Formulation}
\label{s:formulation}

We consider a \jnm{FRMABP} model with $n$ projects \jnm{and a} finite time horizon \jnm{$T > 0$}. Project $i \in [n]$ has state $x_i(t)$ at time \jnm{$t \in [0, T],$} moving over the open state space $\mathcal{X}_i \triangleq (0, H_i)$ with \jnm{$0 < H_i \leqslant \infty$}. We write the system state as $\bm{x}(t) = (x_i(t))_{i=1}^n,$ which belongs to the \jnm{open} state space $\mathcal{X} \triangleq \prod_{i=1}^n \mathcal{X}_i$.
At each time $t$ the system controller chooses a control $\bm{u}(t) = (u_i(t))_{i=1}^n \in [0, 1]^n$, \jnm{which models the level of effort allocated to project $i$ and is required to be piecewise continuous as a function of $t$}. The values $1$ and $0$ \jnm{for $u_i(t)$} represent  ``full effort''  and ``least effort'' level \jnm{allocations for project $i$}, respectively.
At most $m < n$ projects can be set at each time $t$ to the  ``full effort'' level, so  we have  the  coupling resource constraint $\sum_{i=1}^n u_i(t) \leqslant m$ \jnm{for each time $t \in [0, T]$}.

The state dynamics of  project $i$ \jnm{in the FRMABP model} follow \jnm{a} first-order autonomous ordinary differential equation (ODE) referred to as \jnm{the project's} \emph{state equation}: at all
times $t$ where $\bm{u}(\cdot)$ is continuous, for given continuously differentiable functions $\phi_{i}^1(\cdot)$ and $\phi_{i}^0(\cdot),$ 
\begin{equation}
\label{eq:dotxit}
\dot{x}_i(t) = u_i(t) \phi_{i}^1(x_i(t)) + (1 - u_i(t)) \phi_{i}^0(x_i(t)).
\end{equation}
\jnm{Note that} $\phi_{i}^1(\cdot)$ and $\phi_{i}^0(\cdot)$  represent \jnm{project $i$'s state dynamics} under controls $u_i(t) = 1$ and $u_i(t) = 0$, respectively. \jnm{Furthermore, project $i$ accrues rewards over time at the state-and-control dependent rate} $u_i(t)R_{i}^1(x_i(t)) + (1 - u_i(t))R_{i}^0(x_i(t))$, for given continuously differentiable functions $R_{i}^1(\cdot)$ and $R_{i}^0(\cdot)$. 

For a given initial state $\bm{x}_0$, 
\jnm{the FRMABP} can be formulated as the following Problem \eqref{eq:genrmabp}:
\begin{equation}
\begin{aligned}
\label{eq:genrmabp}
&\max_{\bm{x}(\cdot), \bm{u}(\cdot)} &&\int_0^T \sum_{i=1}^n \big[u_i(t) R_{i}^1(x_i(t)) + (1 - u_i(t)) R_{i}^0(x_i(t)) \big] \, dt \\
&\text{subject to:} \quad && \dot{x}_i(t) = u_i(t) \phi_{i}^1(x_i(t)) + (1 - u_i(t)) \phi_{i}^0(x_i(t)) , \quad  \forall i \in [n], t \in [0,T], \\
& \qquad \ &&0 < x_i(t) < H_i, \quad \forall i \in [n],  t \in [0,T],  \\
& \qquad \ &&\bm{x}(0) = \bm{x}_0, \\
& \qquad \ &&0 \leqslant u_i(t) \leqslant 1, \quad \forall i \in [n],  t \in [0,T],  \\
& \qquad \ &&\sum_{i=1}^n u_i(t) \leqslant m, \quad  t \in [0,T].
\end{aligned}
\end{equation}
{\noindent As explained, the goal of Problem (\ref{eq:genrmabp}) is to determine  optimal state and control trajectories, \(\bm{x}^*(\cdot)\) and \(\bm{u}^*(\cdot)\), which maximize the given reward objective while satisfying the prescribed state equations, control constraints, and initial conditions.}

 {\color{black}However, finding trajectories that are provably optimal is known to be notoriously challenging. The widely applied \emph{Pontryagin maximum principle} (PMP), on which we rely, only provides necessary optimality conditions for optimal control problems, unless additional assumptions are incorporated.
While several such assumptions have been identified ensuring optimality of the resulting trajectories (see, e.g., \cite[Ch.\ 2.4]{Sethi2021}), they are  restrictive and not satisfied by many models. In particular, they do not hold for Problem (\ref{eq:genrmabp}), even in the special cases considered herein. 

In such cases, given the elusiveness of proving optimality, it is valuable to compute trajectories satisfying the PMP. Although finding such trajectories is still highly nontrivial, it is a realistic alternative to seeking globally optimal solutions. While they are not guaranteed to be optimal, they are viewed as candidate optimal trajectories. These trajectories, commonly referred to as \emph{extremal}, provide a valuable and widely accepted solution concept in the absence of verifiable global optimality.  This is the approach we pursue here, and henceforth we refer to our computed state and control trajectories and as extremal. For further discussion on this topic, see \citep{survey98, LaValle_2006, pmp22}.
}

In this work, we focus on two fundamental cases that capture many important real-world problems. In the first case, we assume that both the state equation and the reward function \jnm{for each project $i \in [n]$} are affine in the state variable, \jnm{so} $\phi_i^{u}(x_i) = \alpha_i^{u} + \beta_i^{u} x_i$ \jnm{and} $R_i^{u}(x_i) = r_i^{u} x_i - c_i^{u}$ for  $u \in \{0,1\}$.
In the second case, we assume that the state equations are quadratic and the reward functions are affine in the state variable: $\phi_i^{u}(x_i) =\alpha_i^{u}x_i + \beta_i^{u}x_i^2$, with $\alpha_i^{u} \neq 0$ and $\beta_i^{u} \neq 0$, and $R_i^{u}(x_i) = r_i^{u}x_i - c_i^{u}$ for  $u \in \{0,1\}$. 
Note that, for the quadratic dynamics case, we assume that the state equations do not include intercepts in \(\phi_i^{u}\).

\textcolor{black}{Our analysis hinges on two ingredients: (i) closed-form expressions for the state and costate trajectories, and (ii) a representation of the costates as affine functions of suitable state and time dependent primitives. In the two illustrative cases we study, both ingredients are readily available, which lets us present the main ideas and results concisely. Yet our approach extends to broader classes of state equations whenever these two requirements are met.}

For convenience, we introduce the following notation for the remainder of the paper, defined for all \(i \in [n]\):
\begin{equation*}
\begin{aligned}
\alpha_i(u) & \triangleq \alpha_i^{0} + u(\alpha_i^{1} - \alpha_i^{0}), \\
\beta_i(u) & \triangleq \beta_i^{0} + u(\beta_i^{1} - \beta_i^{0}), \\
r_i(u) & \triangleq r_{i}^0 + u(r_{i}^1 - r_{i}^0).
\end{aligned}
\end{equation*}

In general, enforcing state constraints such as $0 < x_i(t) < H_i$ makes the problem considerably more challenging to solve, as the corresponding PMP conditions become more complex (see \cite[Chapter 4]{Sethi2021}). Fortunately, many important problems automatically satisfy such constraints without explicit enforcement. \textcolor{black}{This is because the state equations often guarantee that the state trajectories remain within a prescribed region, drawing on forward invariance / viability results. See, e.g., \cite{aubin2009}.}

Therefore, we focus on the class of problems that can be solved as if such state constraints were not present.

An optimal or extremal solution to Problem (\ref{eq:genrmabp}) can be formally expressed as $$\Big\{\big(\bm{x}^*(t),\bm{u}^*(t)\big)\colon t \in [0,T], \bm{x}^*(0) = \bm{x}_0 \Big\},$$ 
which is a pair of state and control trajectories starting from a \jnm{given} initial state $\bm{x}_0$. {\color{black}In practice, however, one frequently needs to solve Problem (\ref{eq:genrmabp}) for different initial states and remaining time horizons. The solution above only yields the optimal control for the states along the trajectory starting from the given \(\bm{x}_0\); therefore, to obtain the optimal control for any other initial state, one would need to re-solve Problem (\ref{eq:genrmabp}) with the time horizon adjusted accordingly. Since solving Problem (\ref{eq:genrmabp}) is computationally challenging, repeatedly re-solving it is both prohibitive and inefficient.}

Therefore, the mapping from any \jnm{given state-time pair $(\bm{x}, t)$} to a corresponding extremal control $\bm{u}$ is more useful than a \jnm{single state-control trajectory pair starting} from a specific initial state. This mapping is referred to as a \jnm{time-dependent state feedback control} policy.
\jnm{Note that the time dependence of the mapping arises because of the finite horizon} in Problem \eqref{eq:genrmabp}, \jnm{as stationary policies typically do not exist in that case}. Computing such a state feedback control policy, however, is generally considered very challenging. The goal of this work is to propose using OCT-H to learn an extremal time-dependent state feedback policy
$\pi: \mathcal{X} \times [0,T] \mapsto [0,1]^{n}.$
{\color{black} After training, determining the extremal control for any new state and remaining time horizon requires only a fast inference step from the learned policy. } 

\subsection{Contributions}
\label{s:contri}

The contributions of the paper are as follows.
\begin{enumerate}
    \item We initiate the study of \jnm{FRMABPs} where the state equations are either affine or quadratic in the state. We derive fundamental properties of these systems and use them to efficiently implement a numerical solution algorithm known as the shooting method  \citep{StoerJ2013ItNA}.
    \item We propose \jnm{using} the decision tree algorithm OCT-H to learn a \jnm{time-dependent} state feedback policy. To address potential nonlinearities in the training data, we leverage the structural properties of FRMABPs, developing an efficient technique for nonlinear feature augmentation.
    \item We test our approach on machine maintenance, epidemic control, and fisheries control problems, demonstrating that it produces high-quality feedback policies for these applications.
    \item We show that once a policy is learned, it leads to a significant speed-up compared to solving a problem from scratch using the shooting method.
\end{enumerate}

\subsection{Paper Structure}
\label{s:struc}

Section \ref{sec:background} provides the \jnm{PMP} conditions for general FRMABPs and includes a brief review of OCT-H. Section \ref{sec:gendyn} analyzes models with affine and quadratic dynamics. The derived results enable efficient implementation of the shooting method. In Section \ref{sec:ML}, we develop a learning approach based on OCT-H. Section \ref{sec:exp} reports the results of computational experiments, analyzing the accuracy and speed of our approach. In Section \ref{sec:conclusion}, we include our conclusions.

\section{Background}
\label{sec:background}

In this section, we review PMP \citep[Theorem 3.4]{grassetal08} to derive necessary optimality conditions for Problem \eqref{eq:genrmabp}. Following this, we provide a brief overview of OCT-H.

\subsection{PMP Conditions for FRMABPs}
\label{s:genfpaH}

The PMP gives necessary optimality conditions for general optimal control problems \cite[Theorem 3.4]{grassetal08}, which generally are not sufficient. To apply PMP to Problem \eqref{eq:genrmabp}, we  \jnm{consider its} \emph{Hamiltonian}, which \jnm{is a function of the state $\bm{x}$, control $\bm{u}$, and an additional \emph{costate variable} $\bm{y}$,} given by  
\[
\mathcal{H}(\bm{x}, \bm{u}, \bm{y}) = \sum_{i=1}^n \bigg[ u_i R_{i}^1(x_i) + (1 - u_i) R_{i}^0(x_i) + y_i\Big(u_i \phi_{i}^1(x_i) + (1 - u_i) \phi_{i}^0(x_i) \Big) \bigg].
\]
The PMP applied to Problem \eqref{eq:genrmabp} can be formulated as the following lemma.
\jnm{Note that $\mathcal{H}_{x_i}$ denotes the partial derivative of $\mathcal{H}$ with respect to $x_i$.}

\begin{lemma}[Pontryagin maximum principle \citep{Bittner1963LSP, grassetal08}]
\label{lemma:pmp}
If $\bm{x}^*(\cdot)$ and $\bm{u}^*(\cdot)$ are optimal state and control trajectories for Problem $(\ref{eq:genrmabp}),$ then there exists a 
continuous and piecewise continuously differentiable costate \jnm{trajectory}
$\bm{y}(\cdot),$ such that:
\begin{enumerate}[label={\rm (\alph*)}]
\item For all $i \in[n]$, at every time $t$ where $\bm{u}(\cdot)$ is continuous$,$ 
\begin{equation}
\begin{split}
\label{eq:dotyitgen}
\dot{y}_i(t) & = -\mathcal{H}_{x_i}(\bm{x}^*(t), \bm{u}^*(t), \bm{y}(t)) \\
& = -\dot{R}_i^{0}(x_i^*(t)) - y_i(t) \dot{\phi}_i^{0}(x_i^*(t)) \\
& \quad - \big[\dot{R}_i^{1}(x_i^*(t)) - \dot{R}_i^{0}(x_i^*(t))\ + y_i(t) \big(\dot{\phi}_i^{1}(x_i^*(t)) - \dot{\phi}_i^{0}(x_i^*(t))\big)\big] u_i^*(t). 
\end{split}
\end{equation}
\item The following \emph{transversality condition} holds: 
\begin{equation}
\label{eq:yTgen}
\bm{y}(T) = 0.
\end{equation}
\item At each time $t \in [0, T]$, 
\begin{equation}
\label{eq:hxuystarthamax}
\mathcal{H}(\bm{x}^*(t), \bm{u}^*(t), \bm{y}(t)) \geqslant
  \mathcal{H}(\bm{x}^*(t), \bm{u}, \bm{y}(t)) \textup{ for all feasible controls } \bm{u}.  \end{equation}
\end{enumerate}
\end{lemma}

\jnm{As stated in Section \ref{s:formulation}, we shall refer to trajectories satisfying the PMP as \emph{extremal}.}
From {Lemma \ref{lemma:pmp}(c), it follows} that, at any time $t$, state $\bm{x}^*(t)$ and costate $\bm{y}(t)$, an extremal control $\bm{u}^*(t)$ satisfying (\ref{eq:hxuystarthamax}) can be computed by solving the following \emph{linear optimization} (LO) problem:
\begin{equation}
\label{eq:pmpdmLP}
\begin{aligned}
& \max_{\bm{u}} &&\sum_{i=1}^n  \gamma_i(t) \, u_i \\
& \textup{s.t.} \quad && 0 \leqslant u_i \leqslant 1, \quad \forall i \in [n], \\
& \quad && \sum_{i=1}^n u_i \leqslant m,
\end{aligned}
\end{equation}
where $\gamma_i(t) = R_{i}^1(x_i^*(t)) - R_{i}^0(x_i^*(t)) + y_i (t)\big[\phi_{i}^1(x_i^*(t)) - \phi_{i}^0(x_i^*(t))\big]$. 

The next result clarifies the structure of the extremal control trajectories \(\bm{u}^*(\cdot)\) for Problem \eqref{eq:genrmabp}, leveraging Lemma \ref{lemma:pmp}(c). In particular, it shows that at each time \(t\), the corresponding controls allocate full effort (\(u_i^*(t) = 1\)) to a selected subset of projects while assigning no effort (\(u_i^*(t) = 0\)) to the others. This selection is determined by ranking the projects using \(\gamma_i(t)\) as a priority index.

\begin{proposition}
\label{pro:optindxpol}
There exists a piecewise constant extremal control 
$\bm{u}^*(\cdot),$ with $u_i^*(t) \in \{0, 1\}$ for each $i$ and $t$.
{Moreover, at each time \(t\), full effort is assigned to at most \(m\) projects, which are selected as those with the highest nonnegative indices \(\gamma_i(t)\).}
\end{proposition}
\begin{proof} 
As \jnm{discussed above}, for each time $t$, \jnm{an} extremal control $\bm{u}^*(t)$ \jnm{can be} determined by solving LO problem \eqref{eq:pmpdmLP}. 
The $\bm{u}^*(t)$ \jnm{can be chosen to have binary components by using the stated index rule}, which follows straightforwardly from analysis of the dual of problem \eqref{eq:pmpdmLP}:
\begin{equation*}
\label{eq:dpmpdmLP}
\begin{aligned}
& \min_{\bm{v},w} \,&& \sum_{i=1}^n  v_{i} + m w\\
&\mathrm{s.t.} \quad && v_{i} \geqslant 0, \quad w \geqslant 0, \\
&\quad && v_{i} + w \geqslant \gamma_i(t), \quad i \in [n].
\end{aligned}
\end{equation*}

Furthermore, the indices $\gamma^*_i(t)$ are continuous in $t$ by Lemma \ref{lemma:pmp}. As long as the ranking of indices for different projects does not change, the control vector remains constant. Therefore, an extremal control $\bm{u}^*(t)$ can be constructed that is piecewise constant in $t$ and always a binary vector. 
\end{proof}

\begin{remark}
\label{re:{pro:optindxpol}}
\begin{enumerate}[label=(\alph*)]
\item \textcolor{black}{Our results target well-behaved instances in which the PMP delivers a piecewise-constant extremal with only finitely many switches. However, this finiteness is not guaranteed in general: it is well known in optimal control theory that PMP extremals may switch either finitely or infinitely many times, and the latter can even occur over a finite time interval (the \emph{chattering} or \emph{Fuller} phenomenon). Notably, even affine state dynamics do not preclude chattering. See, e.g., \cite{zelikBorisov1994}.}
\item \textcolor{black}{The form of the priority index $\gamma_i(t)$ resembles that of the LP-based index proposed for discrete-state restless bandits in \cite{bnm2000} (see also \cite{gastetal24}) and of the Whittle index \cite{whit88}. 
A key difference is that both those LP indices and the Whittle indices arise from a relaxation that replaces the per-time coupling constraints by a single time-averaged constraint. 
Furthermore, the Whittle indices exist only for \emph{indexable} models and are defined project-wise, depending solely on each project's parameters (hence decoupled from the rest of the system). 
Our indices, instead, are computed from the joint state-costate trajectories of the full system and are therefore inherently coupled and time-dependent, without invoking an averaged-constraint relaxation or indexability.}
\end{enumerate}
\end{remark}

\textcolor{black}{As noted above, Pontryagin extremals are in general not guaranteed to be optimal. To certify optimality one must invoke a sufficiency theorem. A relatively mild condition is given by \cite[Theorem~2.1]{Sethi2021}, which relies on the \emph{derived Hamiltonian}
\[
\mathcal{H}^0(\mathbf{x},\mathbf{y}) \;\triangleq\; \max_{\mathbf{u}\in\mathcal{U}} \, \mathcal{H}(\mathbf{x},\mathbf{u},\mathbf{y}),
\]
where $\mathcal{U}$ is the feasible control set. In our setting, an extremal $\mathbf{u}^*$ (with associated $(\mathbf{x}^*,\mathbf{y})$ satisfying the PMP) is optimal if, for each $t$, the map $\mathbf{x}\mapsto \mathcal{H}^0(\mathbf{x},\mathbf{y}(t))$ is concave. In applications, one often verifies the stronger property that $\mathcal{H}^0(\cdot,\mathbf{y})$ is concave for every $\mathbf{y}$.}

\textcolor{black}{We next show that the models of interest here, despite their simplicity, typically \emph{fail} this condition. Observe that $\mathcal{H}^0(\mathbf{x},\mathbf{y})$ is the optimal value of the linear program}
\[
\max_{\mathbf{u}}\;\sum_{i=1}^n \gamma_i(x_i,y_i)\,u_i
\quad\text{s.t.}\quad
0\le u_i\le 1\ \ (i\in[n]),\qquad
\sum_{i=1}^n u_i \le m,
\]
\textcolor{black}{where}
\[
\gamma_i(x_i,y_i)\;=\; R_i^1(x_i)-R_i^0(x_i) \;+\; y_i\big(\phi_i^1(x_i)-\phi_i^0(x_i)\big).
\]

\textcolor{black}{For clarity, take $m=1$. Then}
\[
\mathcal{H}^0(\mathbf{x},\mathbf{y})
\;=\;
\max\!\Big\{0,\ \max_{i\in[n]} \gamma_i(x_i,y_i)\Big\}.
\]
\textcolor{black}{Even when $R_i^u$ and $\phi_i^u$ are affine in $x_i$, each $\gamma_i(\cdot,y_i)$ is affine, so $\mathcal{H}^0(\cdot,\mathbf{y})$ is a pointwise maximum of affine functions—hence convex, not concave—except in trivial cases where the same control maximizes for all $\mathbf{x}$. Therefore the sufficiency condition fails in the affine case. In the quadratic-dynamics case, $\mathcal{H}^0(\cdot,\mathbf{y})$ becomes a pointwise maximum of (typically) concave quadratics, and a pointwise maximum of concave functions is not concave in general. Consequently, for these models the derived-Hamiltonian condition does not hold, and PMP extremals cannot be asserted to be optimal by this route.}

\textcolor{black}{This limitation affects global and local optimality alike 
---even a unique PMP extremal does not suffice. A proof of local optimality would necessitate the use of advanced second-order conditions from optimal control, which we do not pursue here.}

\subsection{Optimal Classification Trees with Hyperplane Splits}
\label{s:oct}
Optimal Classification Trees (OCT) trains a near-optimal decision tree for classification tasks. Unlike classical decision tree algorithms such as CART \citep{BreiFrieStonOlsh84}, which rely on a greedy algorithm, OCT aims to learn a globally optimal decision tree using advanced optimization techniques. OCT uses a single feature in the node splits, meaning that it partitions the feature space with hyperplanes that are perpendicular to the axes and assigns a prediction to each region. This often results in improved accuracy, robustness to noise, and shallower trees compared to CART.

OCT-H, introduced in \citep{OCT}, is a generalization of OCT, where arbitrary linear combinations of the features are used for splits. Unlike OCT, OCT-H can partition the feature space with arbitrary hyperplanes, enabling it to capture more complex patterns in the data and often leading to better prediction accuracy. For more details on this technique, refer to \citep{MLOPT}.

\section{FRMABPs} 
\label{sec:gendyn}
In this section, we first outline the basic properties of \jnm{extremal control policies for} general FRMABPs. Then, we analyze two distinct cases: one where the state equations are affine in the state and another where they are quadratic. Finally, we describe the shooting method \citep{StoerJ2013ItNA}, a classical numerical algorithm used \jnm{to compute extremal trajectories in optimal control problems}. The version we present is tailored to solve Problem \eqref{eq:genrmabp} using the results we derive.

\subsection{Basic Properties}
\label{s:fhuocp}
Proposition \ref{pro:optindxpol} ensures that for the extremal trajectories \(\bm{u}^*(\cdot)\) and \(\bm{x}^*(\cdot)\), the time interval \([0, T)\) can be partitioned into \(S\) subintervals \([t_s, t_{s+1})\) for \(s = 0, \ldots, S-1\), with \(t_0 = 0\) and \(t_S = T\). Within each subinterval, the extremal control \(\bm{u}(\cdot)\) remains a constant binary vector. In the subsequent sections, we derive the closed-form trajectories of the state and the costate variables in a subinterval with a constant control vector. These closed-form expressions will be used to develop an efficient version of the shooting method, as well as the machine learning approach.

\subsubsection{Affine Dynamics} 
\label{sec:affine}
%A: This is the general alpha-beta model
We next apply the above framework to the case where 
both the state equations and the reward functions are affine in the state, given by $\phi_i^{u}(x_i) = \alpha_i^{u} + \beta_i^{u} x_i$ and
$R_i^{u}(x_i) = r_i^{u} x_i - c_i^{u}$ for $u \in \{0, 1\}, i \in [n]$.

Suppose we are given a finite partition of the time interval $[0, T)$ that consists of $S$ subintervals $[t_s, t_{s+1})$ as above, along with corresponding binary controls $\bm{u}_s \in \{0, 1\}^n$, for $s = 0, 1, \dots, T-1$, with $t_0 = 0$ and $t_S = T$. Given initial state $\bm{x}(0) = \bm{x}_0$ and costate $\bm{y}(0) = \bm{y}_0$ vectors, we consider the trajectories $\{ (\bm{x}(t),\bm{y}(t))\colon t \in [0,T] \}$ obtained by taking control $\bm{u}_s$ on time interval $[t_s, t_{s+1})$ for each $s$. Recall that both the state and  costate trajectories are continuous in $t$, which allows us to build the entire trajectory with the given information. Hence, without loss of generality, we focus on an interval $[t_s, t_{s+1})$ with constant control. 

The ODEs giving the state and costate evolution for project $i$ are: for  $t \in [t_s, t_{s+1}),$ 
\begin{equation}
\label{eq:affinedotxy}
\begin{split}
\dot{x}_i(t)  & = \alpha_i(u_{s,i})
+ \beta_i(u_{s,i})x_i(t), \\
\dot{y}_i(t) & = -r_i(u_{s,i}) - \beta_i(u_{s,i})y_i(t),
\end{split}
\end{equation}
{which are obtained by simplifying the state and costate equations (see (\ref{eq:dotxit}) and (\ref{eq:dotyitgen}) in Lemma \ref{lemma:pmp}(a)) in the affine dynamics case.}

The solution to (\ref{eq:affinedotxy}) on \(t \in [t_s, t_{s+1})\), given the initial conditions \(x_i(t_s)\) and \(y_i(t_s)\), is obtained in closed form via elementary ODE theory, and is given by 
\begin{equation}
\begin{aligned}
x_i(t) &= 
\begin{cases}
\displaystyle x_i(t_s) + \left[ - \frac{\alpha_i(u_{s,i})}{\beta_i(u_{s,i})} - x_i(t_s) \right] \left[1 - e^{\beta_i(u_{s,i})(t-t_s)}\right], & \text{if } \beta_i(u_{s,i}) \neq 0, \\ 
\displaystyle x_i(t_s) + \alpha_i(u_{s,i})(t - t_s), & \text{if } \beta_i(u_{s,i}) = 0, 
\end{cases} \\
y_i(t) &= 
\begin{cases}
\displaystyle y_i(t_s) + \left[ - \frac{r_i(u_{s,i})}{\beta_i(u_{s,i})} - y_i(t_s) \right] \left[1 - e^{-\beta_i(u_{s,i})(t-t_s)}\right], & \text{if } \beta_i(u_{s,i}) \neq 0, \\ 
\displaystyle y_i(t_s) - r_i(u_{s,i})(t - t_s), & \text{if } \beta_i(u_{s,i}) = 0. 
\end{cases}
\end{aligned}
\label{eq:affineex}
\end{equation}

\subsubsection{Quadratic Dynamics}
\label{sec:quadratic}
%A: This is the logistic model (fishery, epidemiology)

We proceed similarly \jnm{in} the case where the state equations and the reward functions are quadratic \textcolor{black}{with no intercept} and affine in the state, respectively. That is, the reward function for each project $i \in [n]$ is given by  $R_{i}^u(x_i) = r_i^{u} x_i - c_i^{u}$ and the state equation is  $\phi_i^{u}(x_i) = \alpha_i^{u}x_i + \beta_i^{u} x_i^2$, with $\alpha_i^{u} \neq 0$ and $\beta_i^{u} \neq 0$  for $u \in \{0,1\}$. 
\textcolor{black}{Note that we do not include an intercept term in $\phi_i^{u}(x_i)$ so as to obtain a relatively simple closed-form solution for the state trajectories (see (\ref{eq:quadex})), which is not the case if an intercept is included. The no-intercept setting also justifies treating the quadratic and the affine dynamics cases separately, along with the fact that the resulting formulae for the respective state and costate trajectories are substantially different.}
Although not addressed in this work, generalizing the reward functions to quadratic forms is straightforward. Again, we focus on the fixed interval $[t_s, t_{s+1})$ with constant control $\bm{u}(t) = \bm{u}_s$. In this interval, the ODEs governing the state and costate evolution for project $i$ are:
\begin{equation}
\label{eq:quad}
\begin{split}
\dot{x}_i(t) & = \alpha_i(u_{s,i})x_i(t) + \beta_i(u_{s,i})x_i^2(t), \\
\dot{y}_i(t) & = -r_i(u_{s,i}) - y_i(t)\big[\alpha_i(u_{s,i}) + 2\beta_i(u_{s,i})x_i(t) \big],
\end{split}
\end{equation}
{which results from simplifying the state and costate equations (see (\ref{eq:dotxit}) and (\ref{eq:dotyitgen}) in Lemma \ref{lemma:pmp}(a)) in the quadratic dynamics case.}

The solution to (\ref{eq:quad}) is given as follows. For brevity, we omit the trivial case where $\dot{x}_i(t) = 0$. Unlike the case of affine dynamics, we do not use expressions involving the initial conditions in the subinterval $x_i(t_s)$ and $y_i(t_s)$, as this leads to considerably more complicated expressions. Instead, we introduce constants $K \textcolor{black}{\neq 0}$ and $G$ for simplicity\textcolor{black}{:}
\begin{equation}
\label{eq:quadex}
\begin{split}
x_i(t) & = \frac{K\alpha_i(u_{s,i})e^{\alpha_i(u_{s,i})t}}{1 - K\beta_i(u_{s,i})e^{\alpha_i(u_{s,i})t}} \\
 y_i(t) & = G e^{-\alpha_i(u_{s,i})t}(1 - K \beta_i(u_{s,i})e^{\alpha_i(u_{s,i})t})^2 \\
 &\quad - \frac{r_i(u_{s,i})}{K\alpha_i(u_{s,i})\beta_i(u_{s,i})}(1 - K \beta_i(u_{s,i})e^{\alpha_i(u_{s,i})t})e^{-\alpha_i(u_{s,i})t}.
\end{split}
\end{equation}

% \textcolor{black}{NOTE: in the first equation, one gets the following expression for $K_{s,i}$:
% \[
% K_{s,i} = \frac{x_i(t_s) e^{-\alpha(u_{s,i})t_s}}{\alpha_i(u_{s,i}) 
% + \beta_i(u_{s,i}) x_i(t_s)},
% \]
% which is only valid if $\alpha_i(u_{s,i}) + \beta_i(u_{s,i}) x_i(t_s) \neq 0$. For this reason, I propose to get rid of the $K_{s,i}$ and keep the given exact expression in black.}

% \textcolor{black}{NOTE: As for the second equation, it is more convenient to solve for $y_i(t)$ in terms of $x_i(t)$, which gives
% \[
%  y_i(t) = \bigg[-\frac{r_i(u_{s,i})}{\beta_i(u_{s,i})}+\frac{\alpha_i^4(u_{s,i}) L_{s, i}}{\alpha_i(u_{s,i}) + \beta_i(u_{s,i}) x_i(t)}\bigg] x_i^{-1}(t)
% \]
% From 
% \[
% y_i(t_s) = \bigg[-\frac{r_i(u_{s,i})}{\beta_i(u_{s,i})}+\frac{\alpha_i^4(u_{s,i}) L_{s, i}}{\alpha_i(u_{s,i}) + \beta_i(u_{s,i}) x_i(t_s)}\bigg] x_i^{-1}(t_s)
% \]
% one gets
% \[
% L_{s, i} = \frac{[\alpha_i(u_{s,i}) + \beta_i(u_{s,i}) x_i(t_s)] [r_i(u_{s,i})+ \beta_i(u_{s,i}) x_i(t_s) y_i(t_s)]}{\alpha_i^4(u_{s,i}) \beta_i(u_{s,i}) }
% \]
% and hence
% \[
%  y_i(t) = \bigg(-\frac{r_i(u_{s,i})}{\beta_i(u_{s,i})}+\frac{[r_i(u_{s,i})+ \beta_i(u_{s,i}) x_i(t_s) y_i(t_s)] [\alpha_i(u_{s,i}) + \beta_i(u_{s,i}) x_i(t_s)]}{\beta_i(u_{s,i}) [\alpha_i(u_{s,i}) + \beta_i(u_{s,i}) x_i(t)]}\bigg) x_i^{-1}(t)
% \]
% but we have to clarify what happens when $\alpha_i(u_{s,i}) + \beta_i(u_{s,i}) x_i(t) = 0$, in which case $\dot{x}_i(t) = 0$
% }

\subsection{The Shooting Method}
\label{s:affineacot}
We use the results derived above to design an algorithm for computing extremal trajectories for Problem \eqref{eq:genrmabp} with a fixed initial state $\bm{x}_0$. The algorithm is based on the well-known \emph{shooting method}, and further exploits the structure of the extremal trajectories identified \jnm{above}. 

{\color{black}The algorithm aims to find an initial costate value \(\bm{y}_0\) such that the trajectory initiated from the initial conditions \((\bm{x}_0, \bm{y}_0)\) satisfies the terminal condition \(\bm{y}(T)=0\). Once \(\bm{x}_0\) and \(\bm{y}_0\) are fixed, the corresponding trajectories \((\bm{x}(\cdot), \bm{u}(\cdot), \bm{y}(\cdot))\) up to time \(T\) can be computed using the derivations in Section \ref{s:fhuocp}. In particular, the closed-form expressions for \(\bm{x}(\cdot)\) and \(\bm{y}(\cdot)\), which satisfy \eqref{eq:dotxit} and Lemma \ref{lemma:pmp}(a) respectively, allow us to propagate the state and costate. The control trajectory \(\bm{u}(\cdot)\) is determined by the index policy that satisfies \eqref{eq:hxuystarthamax}, which is derived to meet the condition in Lemma \ref{lemma:pmp}(c). Once a trajectory is found that satisfies \(\bm{y}(T)=0\) (Lemma \ref{lemma:pmp}(b)), it meets all the necessary conditions of Lemma \ref{lemma:pmp} and is therefore classified as extremal.

Specifically, at each time $t$, given $\bm{x}(t)$ and $\bm{y}(t)$, we rank the index functions $\gamma_i(t)$ for all projects to determine the associated control vector. As long as the control vector remains unchanged, the trajectories are propagated using the closed-form expressions derived in Equations \eqref{eq:affineex} and \eqref{eq:quadex}. A change in the control vector indicates the start of a new subinterval. Because \(\bm{x}(\cdot)\) and \(\bm{y}(\cdot)\) are continuous in time, the terminal values of one subinterval serve as the initial conditions for the next. According to Lemma \ref{lemma:pmp}, if the computed trajectories satisfy \(\bm{y}(T) = 0\), they are guaranteed to be extremal (i.e., they meet the necessary optimality conditions).}

Formally, let $\bm{g}:\mathbb{R}^n \mapsto \mathbb{R}^n$ be the function that takes an initial costate value $\bm{y}_0$ as an input, and the resulting terminal costate value $\bm{y}(T)$ as an output. The shooting method is an iterative algorithm to solve the $n \times n$ root finding problem $\bm{g}(\bm{y}) = \bm{0}$ numerically.

In the $k_{th}$ iteration of the shooting method, the algorithm starts with a guess $\bm{y}_{k, 0}$ for $\bm{y}(0)$, computes corresponding state and costate trajectories $\bm{x}(\cdot)$ and 
$\bm{y}(\cdot)$ up to time $T$ following the index policy structure. This automatically partitions of the time interval $[0, T)$ into subintervals $[t_s, t_{s+1})$ with a constant control $\mathbf{u}_s$, for $s = 0, \ldots, S-1$, 
with $t_0 = 0$ and $t_S = T$, where the number $S$ of intervals is also determined. Then, the algorithm checks whether $\bm{y}(T) \approx 0$ within a given tolerance level $\epsilon$. If such is the case, the algorithm stops. Otherwise, the initial costate value is updated and the process is repeated. The update follows Broyden's method \citep{broyden65, gay79}, a well-established derivative-free quasi-Newton's method.

In the $k_{th}$ iteration, $k \geqslant 1$, Broyden's method computes iterates $\bm{y}_{k,0}$ and $\bm{J}_k$, a surrogate for the Jacobian of $\bm{g}$, following 
\begin{equation}\label{eq:broyden}
\begin{aligned}
\bm{y}_{k,0} & = \bm{y}_{k-1,0}  - \big(\bm{J}_{k-1}\big)^{-1} \bm{g}(\bm{y}_{k-1,0}), \\
\bm{J}_k & = \bm{J}_{k-1} + \frac{\bm{g}(\bm{y}_{k,0}) - \bm{g}(\bm{y}_{k-1,0}) - \bm{J}_{k-1} (\bm{y}_{k,0} - \bm{y}_{k-1,0})}{\lVert{\bm{y}_{k,0} - \bm{y}_{k-1,0}}\rVert^2} \big(\bm{y}_{k,0} - \bm{y}_{k-1,0}\big)'.
\end{aligned}
\end{equation}
Typically, the method starts with $\bm{J}^0 = \bm{I}$, the identity matrix. Also, when we compute the trajectory starting from $t = 0$, we choose a small step size $\delta$ and proceed by incrementally adding $\delta$ to $t$ until reaching $T$. The details of the algorithm are provided in Algorithm \ref{alg:shooting}.

\textcolor{black}{Note that Algorithm \ref{alg:shooting} assumes that, for a given instance, the number of subintervals $S$ is finite, i.e., there is no chattering. If Algorithm \ref{alg:shooting} terminates under this assumption, the resulting trajectory is still guaranteed to satisfy the PMP conditions.}

\begin{remark}
\label{re:shooting}
In general, shooting methods are known to be vulnerable to numerical instability. This instability occurs primarily because the shooting methods involve computing the trajectories of $\bm{x}(\cdot),\bm{y}(\cdot)$, which often requires numerical approximations of  solutions to differential equations \citep{Diek76, khal92, StoerJ2013ItNA}. However, for the models with affine and quadratic dynamics discussed in this work, no approximation is required because closed-form expressions for $\bm{x}(\cdot),\bm{y}(\cdot)$ exist over any interval with a constant control, given initial values. This advantage reduces the risk of numerical errors and enhances the reliability and accuracy of the shooting method.
\end{remark}

\begin{algorithm}
 \KwInput{$\epsilon, \delta, R^0_i(\cdot), R^1_i(\cdot), \phi^0_i(\cdot), \phi^1_i(\cdot), T, m, n, \bm{x}_0$ }
\KwOutput{$\Big\{\big( \bm{x}^*(t),\bm{u}^*(t), \bm{y}(t)\big): t \in [0,T], \bm{x}(0) = \bm{x}_0 \Big\}$ } 
\textbf{Initialization: $\bm{y}_{0,0} \in \mathbb{R}^{n}, \bm{J}_{0} =\bm{I}, k = 0$} 

\vspace{2mm}

\Repeat{$\|\bm{g}(\bm{y}_{k-1,0})\|_\infty \leq {\epsilon}$}{
    Fix the initial state to $\bm{x}_0$ and the initial costate to $\bm{y}_{k,0}$. \\
    Compute the trajectories of $\bm{x}(t),\bm{u}(t), \bm{y}(t)$ starting from $t = 0$ up to $t = T$ using \eqref{eq:pmpdmLP}, and \eqref{eq:affineex} or \eqref{eq:quadex}. \\
    $k \leftarrow k + 1$. \\
    Update $\bm{y}_{k,0}$ and $\bm{J}_k$ as in \eqref{eq:broyden}.
}

\textbf{Return} $\Big\{\big( \bm{x}(t),\bm{u}(t), \bm{y}(t)\big): t \in [0,T] , \bm{x}(0) = \bm{x}_0 \Big\}$.
  \caption{Shooting Method}
 \label{alg:shooting}
\end{algorithm}

\section{A Machine Learning Approach}
\label{sec:ML}
%A: This is where we present ML-related theorems. Adam's part.

In this section, we present a machine learning approach to learn a \jnm{time-dependent} state feedback policy for Problem \eqref{eq:genrmabp}. We begin by providing an overview of the algorithm, and then introduce a nonlinear feature augmentation technique to incorporate nonlinearity into the learned policy. Finally, we present an example of our approach, focusing on the case of admission and routing to parallel infinite-server queues.

\subsection{Algorithm Overview}

Our approach uses OCT-H to imitate the time-dependent state-control mappings in the extremal trajectories. To achieve this, we generate multiple initial states for Problem \eqref{eq:genrmabp} and solve each instance using Algorithm \ref{alg:shooting}. For each instance with an initial state $\bm{x}_0$, we obtain the extremal state and control trajectories $\Big\{ (\bm{x}^*(t), \bm{u}^*(t)): t \in [0,T], \bm{x}(0) = \bm{x}_0\Big\}$. We select $N \in \mathbb{N}$ distinct time points $t_1, \dots, t_N$ from the interval $[0,T]$, and extract the corresponding state values $\bm{x}^*(t_1), \dots, \bm{x}^*(t_N)$ and the control values $\bm{u}^*(t_1), \dots, \bm{u}^*(t_N)$. The training data extracted from this process consists of $\bigg\{\Big(\big(\bm{x}^*(t_\ell),t_\ell\big), \bm{u}^*(t_\ell)\Big)\bigg\}_{\ell=1}^{N}$, where the tuple $\big(\bm{x}^*(t_\ell), t_\ell\big)$ is the feature vector and $\bm{u}^*(t_\ell)$ is the target for the $\ell_{th}$ data point. The time $t_{\ell}$ is included as part of the feature vector to incorporate the time-dependency of the state feedback policy.

As discussed in Section \ref{s:oct}, OCT-H partitions the state space using hyperplanes and assigns predictions to each region. Consequently, OCT-H may struggle to capture complex nonlinear patterns in the data. One straightforward solution could be to use deep neural networks, leveraging their strong approximation capabilities. However, deep neural networks are black-box algorithms that obscure how decisions are made, as they automatically learn nonlinearities embedded in their layers. Moreover, simply feeding raw data into a deep neural network makes it challenging to fully exploit the inherent structure of FRMABPs.

Instead of using deep neural networks, we adhere to OCT-H and address potential nonlinearities in the decision boundaries by augmenting the feature vector with nonlinear transformations of the state variables. Specifically, for each feature vector $\big(\bm{x}^*(t_\ell),t_\ell\big)$ in the training data, we add nonlinear transformations of the state vector $\bm{x}^*(t_\ell)$ as additional entries. Then, we apply OCT-H to the augmented data set. This approach ensures that the resulting policy remains interpretable, maintaining the decision-making process within the framework of decision trees. 

{\color{black}After the training phase, also referred to as the ``offline'' phase, we can deploy the learned policy to compute the extremal control for any new state and its corresponding remaining time horizon. This deployment is referred to as the ``online'' phase. We provide a graphical overview our entire pipeline in Figure \ref{fig: diagram}.}

\begin{figure}[t]
    \centering
    \includegraphics[width=\linewidth,height=1.5in]{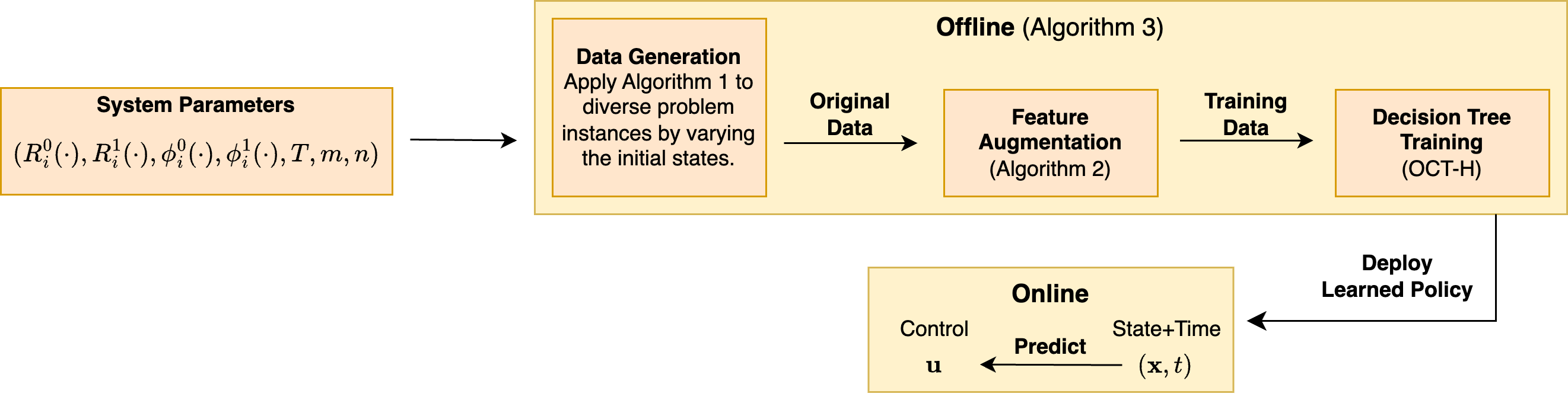}
    \caption{Overview of the proposed method. With fixed system parameters, Algorithm \ref{alg:shooting} is applied to diverse problem instances to generate a dataset. This dataset is then augmented using Algorithm \ref{alg:augmentation}. Then, OCT-H is used to train a time-dependent state feedback policy. After training is complete, the policy is deployed to predict the corresponding control vector for any new state–time pair.}
    \label{fig: diagram}
\end{figure}

\subsection{Addressing Nonlinearity by Feature Augmentation}

We propose a heuristic approach to choosing nonlinear transformations on the state variables. As \jnm{discussed in Section \ref{s:genfpaH} (see Proposition \ref{pro:optindxpol}), an}  extremal control at time 
$t$ is determined by the relative ordering of the priority indices $\gamma_i(t), i \in [n]$. This implies that switches in effort (i.e., changes in the control vector) occur either when $\gamma_i(t) = \gamma_j(t)$ for some projects $i \neq j$, or when $\gamma_i(t) = 0$. At these points, it becomes indifferent whether we invest effort in project $i$ or $j$, or whether we invest effort in project $i$ or idle, respectively. The set of points in the \jnm{time-state space} where such switching of efforts occur are known as  \emph{switching curves} in the optimal control literature \citep{Sethi2021}. Outside of these curves in the \jnm{time-state space}, the extremal vector is unique and constant. From a supervised learning perspective, the switching curves act as discriminating boundaries in the feature space that separate data points with different labels. However, identifying switching curves or even their functional forms is highly nontrivial. Switching curves can be found by expressing the equalities $\gamma_i(t) = \gamma_j(t)$ only with respect to the state variables $\bm{x}(t)$ and time $t$, yet $\gamma_i(t)$ also depends on the costate variable $y_i(t)$. This means that we must express $y_i(t)$ in terms of $\bm{x}(t)$ and $t$. We propose a heuristic to approximate the relation between $y_i(t)$,$\bm{x}(t)$, and $t$.

\subsubsection{Affine Dynamics}
\label{s:adh}
Consider an interval $[t_s, t_{s+1})$ with a constant control $\bm{u}_s$. We first express the index function $\gamma_i(t)$ only in terms of $x_i(t)$ in this specific interval. 

\begin{proposition}
\label{pro:switching_affine}
For the FRMABP with affine dynamics, in an interval $[t_s, t_{s+1})$ with a constant control $\bm{u}_s$, $\gamma_i(t)$ can be expressed as an affine combination of the terms
\begin{equation}
t, x_i(t), x_i^2(t), \frac{1}{{x_i(t) + \frac{\alpha_i^{0}}{\beta_i^{0}}}}, \frac{1}{{x_i(t) + \frac{\alpha_i^{1}}{\beta_i^{1}}}}.
\label{eq:affine_features}
\end{equation}
\end{proposition}

\begin{proof} 
Using the results in \eqref{eq:affineex}, the relation between $x_i(t)$ and $y_i(t)$ can be described as follows:
\begin{equation*}
y_i(t) = 
\begin{cases}
\begin{aligned}
\displaystyle y_{i,1}(t,u_{s,i}) \triangleq & \; -\frac{r_i(u_{s,i})}{\beta_i(u_{s,i})} + 
\frac{\left( x_i(t_s) + \frac{\alpha_i(u_{s,i})}{\beta_i(u_{s,i})} \right) \left( y_i(t_s) + \frac{r_i(u_{s,i})}{\beta_i(u_{s,i})} \right)}{x_i(t) + \frac{\alpha_i(u_{s,i})}{\beta_i(u_{s,i})}}, \\
& \; \text{if } \beta_i(u_{s,i}) \neq 0, 
\end{aligned} \\[3ex]
\begin{aligned}
\displaystyle y_{i,2}(t,u_{s,i}) \triangleq & \; y_i(t_s) - r_i(u_{s,i}) \frac{x_i(t) - x_i(t_s)}{\alpha_i(u_{s,i})}, \\
& \; \text{if } \beta_i(u_{s,i}) = 0, \alpha_i(u_{s,i}) \neq 0, r_i(u_{s,i}) \neq 0, 
\end{aligned} \\[2ex]
\begin{aligned}
\displaystyle y_{i,3}(t,u_{s,i}) \triangleq & \; x_i(t) - x_i(t_s) + y_i(t_s) - r_i(u_{s,i})(t - t_s), \\
& \; \text{if } \beta_i(u_{s,i}) = \alpha_i(u_{s,i}) = 0, r_i(u_{s,i}) \neq 0, 
\end{aligned} \\[2ex]
\begin{aligned}
\displaystyle y_{i,4}(t,u_{s,i}) \triangleq & \; y_i(t_s), \\
& \; \text{if } \beta_i(u_{s,i}) = r_i(u_{s,i}) = 0. 
\end{aligned}
\end{cases}
\end{equation*}

\noindent Given that the extremal control vector is always binary, as stated in Proposition \ref{pro:optindxpol}, this relation leads to the unified expression:
\begin{equation}\label{eq:combined}
\begin{aligned}
y_i(t) &=  \sum_{j=1}^{4}\bigg(C_{1,j}y_{i,j}(t,1) + C_{0,j}y_{i,j}(t,0)\bigg),
\end{aligned}
\end{equation}

\noindent where only one of the coefficients $C_{u,j}, u \in \{0,1\}, j \in [4]$, is non-zero, depending on the values of $\alpha_i(u_{s,i}), \beta_i(u_{s,i}), r_i(u_{s,i})$.

Subsequently, we substitute the expression \eqref{eq:combined} into the index function 
$$\gamma_i(t) = r_{i}^1x_i(t) - c_i^{1} - r_{i}^0x_i(t) + c_i^{0} + y_i(t)[\alpha_i^{1} + \beta_i^{1}x_i(t) - \alpha_i^{0} - \beta_i^{0}x_i(t)],$$ resulting in an expression for $\gamma_i(t)$ that involves an affine combination of the terms only with respect to $x_i(t)$ given in \eqref{eq:affine_features}. Here, the terms $\frac{1}{{x_i(t) + \frac{\alpha_i^{0}}{\beta_i^{0}}}}$ and $\frac{1}{{x_i(t) + \frac{\alpha_i^{1}}{\beta_i^{1}}}}$ correspond to $y_{i,1}(t,u_{s,i})$, while $x_i^2(t)$ corresponds to $y_{i,2}(t,u_{s,i})$ and $y_{i,3}(t,u_{s,i})$, and $t$ corresponds to $y_{i,3}(t,u_{s,i})$.  
\end{proof}

The initial values $y_i(t_s)$ and $x_i(t_s)$ that appear in the coefficients of the affine combination in Proposition \ref{pro:switching_affine} should vary depending on the specific time interval and the initial state $\bm{x}(0)$ that led to that interval. In other words, these coefficients remain constant only within the specific interval considered. Thus, these values should be treated also as time and state-dependent, rather than as global constants. This means that the relation between $\gamma_i(t)$ and $x_i(t)$ described in Proposition \ref{pro:switching_affine} does not globally apply across the entire state-time space.

We propose simply approximating $\gamma_i(t)$ with the affine combination of the terms in \eqref{eq:affine_features}, regardless of time and state. This approach treats the coefficients of the affine combination as constants to be learned, and allows us to approximate any switching curve using the affine combination of the terms in \eqref{eq:affine_features}, defined for all $i \in [n]$. Therefore, for each $i \in [n]$, we augment the original feature vectors with the terms in  \eqref{eq:affine_features}. Applying OCT-H to these augmented feature vectors enables the learning of nonlinear switching curves within the state space. 

Note that while our approach assumes implicit time-dependence of switching curves through the inclusion of time as part of the feature vector in each original data point, explicit time dependencies are not separately modeled in the approximation of $\gamma_i(t)$.

%%%%%%%%%

%%%%%%%%%

\subsubsection{Quadratic Dynamics}
\label{s:qdh}
For the quadratic case, we apply the same principle as in the affine dynamics, leading to the following proposition. 

\begin{proposition}
\label{pro:switching_quad}

For RMABPs with quadratic dynamics, in an interval $[t_s, t_{s+1})$ with a constant control $\bm{u}_s$, $\gamma_i(t)$ can be expressed as an affine combination of the terms
\begin{equation}
x_i(t), \frac{1}{x_i(t)}, \frac{1}{{x_i(t) + \frac{\alpha_i^{0}}{\beta_i^{0}}}}, \frac{1}{{x_i(t) + \frac{\alpha_i^{1}}{\beta_i^{1}}}}.
\label{eq:quad_features}
\end{equation}
\end{proposition}

\begin{proof} 
The results in \eqref{eq:quadex} lead to the following expression:
$$
y_i(t) = \Bigg( \frac{K(\alpha_i(u) + x_i(t)\beta_i(u))}{x_i(t)} - K\beta_i(u) \Bigg) \Bigg(G - G\beta_i(u)\frac{x_i(t)}{\alpha_i(u) + x\beta_i(u)} - \frac{r_i(u)}{K\alpha_i(u)\beta_i(u)} \Bigg)
$$
We plug this expression into the following index function to get the results:
$$
\gamma_i(t) = u(r^1_ix_i(t) - c_i^1) + (1-u)(r^0_ix_i(t) - c_i^0) + y_i(t)\Bigg(u(\alpha_i^{1}x_i + \beta_i^{1} x_i^2) + (1-u)(\alpha_i^{0}x_i + \beta_i^{0} x_i^2) \Bigg),
$$
\end{proof}

\noindent Following the approach used for affine dynamics, we augment the original feature vector in each data point with the terms in \eqref{eq:quad_features}, for all $i \in [n]$.

\subsubsection{Data-driven Feature Augmentation}
\label{s:ddfs}
In practice, not all of the terms listed above are needed, as the expression of $y_i(t)$ with respect to $x_i(t)$ depends on the problem parameters ${\alpha}_{u_{s,i},i}, {\beta}_{u_{s,i},i}, {r}_{u_{s,i},i}$, and on whether $u_{s,i}$ is 1 or 0. This means some of the expressions in \eqref{eq:affine_features} and \eqref{eq:quad_features} may not be used at all. After we have generated the training data, we can augment the feature vector in a data-driven way, by investigating the control variables that occurred in the generated data. For example, in the models with affine dynamics, assume $u_i = 1$ for all data points generated. If $\beta_i^{1} \neq 0$, then the terms resulting from the cases $y_{i,1}(t,0)$ and $y_{i,j}(t,u_i), j = 2,3,4 $, are unnecessary because only $C_{1,1}$ in \eqref{eq:combined} will be non-zero. Hence, we augment the feature vector only with $\frac{1}{{x_i(t) + \frac{\alpha_i^{1}}{\beta_i^{1}}}}$. This principle generalizes to data-driven feature augmentation process, provided in Algorithm \ref{alg:augmentation}. We describe the entire offline phase, beginning from data generation to decision tree training, in Algorithm \ref{alg:main}.

\begin{algorithm}
\KwInput{$\mathcal{D}, \mathcal{U}, n$}
\KwOutput{Augmented feature vectors $\mathcal{D}^{'}$} 
\textbf{Initialization: $ \mathcal{D}^{'} = \mathcal{D}$}

Identify $v_i = \{u_i \in \mathbb{R}: \bm{u} \in \mathcal{U} \}, \; \forall i \in [n]$.\\

\textbf{Affine Dynamics} \\ 
\For{$i \in [n]$}{
  \For{$u_i \in v_i$}{
    \For{$ \big(\bm{x},t\big) \in \mathcal{D}^{'}$}{
        \If{$\beta_i(u_{i}) \neq 0$}{Add $\frac{1}{x_i + \frac{\alpha_i({u_i})}{\beta_i({u_i})}}$ as an additional entry to the feature vector $\big(\bm{x},t\big)$.}
        \ElseIf{$r_i(u_i) \neq 0$}{Add $x_i^2$ as an additional entry to the feature vector $\big(\bm{x},t\big)$.}
    }
  }
}

\vspace{3mm}

\textbf{Quadratic Dynamics} \\

\For{$i \in [n]$}{
  \For{$ \big(\bm{x},t\big) \in \mathcal{D}^{'}$}{
          Add $\frac{1}{x_i}$ as an additional entry to the feature vector $\big(\bm{x},t\big)$. \\
    \For{$u_i \in v_i$}{
         Add $\frac{1}{x_i + \frac{\alpha_i({u_i})}{\beta_i({u_i})}}$ as an additional entry to the feature vector $\big(\bm{x},t\big)$.
    }
  }
}

\caption{Data-Driven Feature Augmentation}
\label{alg:augmentation}
\end{algorithm}

\begin{algorithm}
\KwInput{$M, \{t_1, \dots, t_N\}, \epsilon, \delta, R^0_i(\cdot), R^1_i(\cdot), \phi^0_i(\cdot), \phi^1_i(\cdot), T, m, n, \bm{x}_0$}
\KwOutput{$
\pi: \mathcal{X} \times [0,T] \mapsto [0,1]^{n}
$} 

\textbf{Initialization: $ \mathcal{D} = \emptyset, \mathcal{U} = \emptyset,  j = 1$} 

\vspace{3mm}

\textbf{1. Data Generation} \\

\For{$j = 1, \dots, M$}{
    Sample an initial state ${\bm{x}_0}$ from the state space $\mathcal{X}$. \\
    Solve Problem \eqref{eq:genrmabp} with the initial state $\bm{x}_0$ using Algorithm \ref{alg:shooting}. \\
    $\mathcal{D} \leftarrow \mathcal{D} \cup \bigg\{\big(\bm{x}^*(t_\ell),t_\ell\big)\bigg\}_{\ell=1}^{N}$ \\
    $\mathcal{U} \leftarrow \mathcal{U} \cup \bigg\{\bm{u}^*(t_\ell)\bigg\}_{\ell=1}^{N} $ \\
    $j \leftarrow j + 1$
}

\vspace{3mm}

\textbf{2. Feature Augmentation } \\

Use Algorithm \ref{alg:augmentation} to augment the feature vectors in $\mathcal{D}$ and denote the augmented feature vectors as $\mathcal{D}^{'}$.

\textbf{3. Training} \\
Use OCT-H to train a classification tree. The feature vectors are $\mathcal{D}^{'}$ and the target vectors are $\mathcal{U}$. 

\caption{OCT-H for Fluid Restless Multi-armed Bandits.}
\label{alg:main}
\end{algorithm}

\subsection{Example: Optimal Control of Admission and Routing to Parallel Infinite-Server Queues}
\label{s:cocarpisq}
%A: This is where we present the infinite server routing, where we know the closed form solution

We study the optimal control of admission and routing to parallel infinite-server queues as a special case of models with affine dynamics. Due to its simple structure, we can derive closed-form expressions for the index functions $\gamma_i(t), i\in [n]$. For small $n$, this allows us to obtain a simple closed-form extremal policy. Then, we use Algorithm \ref{alg:main} to learn a decision tree policy and compare it with the closed-form extremal policy.

\subsubsection{Deriving the Closed-form Index Functions}
\label{s:queueingindex}

Consider a system with $n$ parallel fluid queues with infinite buffers. 
Fluid arrives to the system at rate $\lambda$. 
The controller chooses the proportion $u_i(t) \in [0, 1]$ to be routed to each queue $i$ at each time $t$. 
The system equation for the buffer contents $x_i(t)$ of queue $i$ is 
\begin{equation*}
\label{eq:dotxitarpq}
\dot{x}_i(t) = 
\lambda u_i(t) - \mu_i x_i(t),
\end{equation*}
which corresponds to the fluid analog of an infinite-server queue.
The remaining proportion, $1 - \sum_{i=1}^n u_i(t)$, is rejected, incurring a cost rate $R$. Note that $\sum_{i=1}^n u_i(t) \leqslant 1$.
Furthermore, queue $i$ incurs holding costs at rate $C_i$. 

The goal is to  minimize the following cost objective over a finite horizon $[0, T]$:
\[
\min_{\bm{u}(\cdot), \bm{x}(\cdot)} \, \int_0^T \big[R \lambda (1 - \sum_i u_i(t)) + \sum_i C_i x_i(t) \big] \, dt,
\]
which in turn is reformulated in maximization form as 
\begin{equation*}
\label{eq:ocparpq}
\max_{\bm{u}(\cdot), \bm{x}(\cdot)} \, \int_0^T \sum_i [R \lambda u_i(t) - C_i x_i(t)] \, dt.
\end{equation*}
The upper bound in the state space of each queue $i$ is $H_i \triangleq \infty$. We define the optimal control of admission and routing to parallel infinite-server queues as the following problem:

\begin{equation}
\begin{aligned}
\label{eq:routing}
&\max_{\bm{u}(\cdot), \bm{x}(\cdot)} &&\int_0^T \sum_{i=1}^{n} [R \lambda u_i(t) - C_i x_i(t)] \, dt \\
&\text{s.t.} \quad && \dot{x}_i(t) = \lambda u_i(t) - \mu_i x_i(t), \quad  \forall i \in [n], \forall t \in [0,T], \\
& \qquad \ && x_i(t) > 0, \quad \forall i \in [n], \forall t \in [0,T],  \\
& \qquad \ &&\bm{x}(0) = \bm{x}_0, \\
& \qquad \ &&0 \leqslant u_i(t) \leqslant 1, \quad \forall i \in [n], \forall t \in [0,T],  \\
& \qquad \ &&\sum_{i=1}^n u_i(t) \leqslant 1, \quad \forall t \in [0,T].
\end{aligned}
\end{equation}

We first prove that the constraints $x_i(t) > 0, \forall i \in [n]$, are automatically satisfied regardless of the control trajectory.

\begin{proposition}
\label{pro:routingpositive}

For Problem \eqref{eq:routing}, if $\bm{x}_0 > 0$, then $\bm{x}(t) > 0, \forall t \in [0,T]$ regardless of the control trajectory.
\end{proposition}

\begin{proof} 
For $t \in [t_s, t_{s+1})$ with a constant control vector $\bm{u}_s \in [0,1]^{n}$, 
\begin{equation*}
\begin{split}
{x}_i(t)  & = x_i(t_s)+ [\frac{\lambda u_{s,i} }{\mu_i} - x_i(t_s)] \, \big[1 - e^{-\mu_i 
 (t-t_s)}\big] \\
& = \frac{\lambda u_{s,i} }{\mu_i}(1 - e^{-\mu_i 
 (t-t_s)}) + x_i(t_s)e^{-\mu_i 
 (t-t_s)}
\end{split}
\end{equation*}
Assuming $x_i(t_s) > 0$, it is straightforward that $x_i(t)  > 0$ in the interval $[t_s, t_{s+1})$. It is also clear that $x_i(t_{s+1})  > 0$, again implying $x_i(t)  > 0$ in the interval $[t_{s+1}, t_{s+2})$.  Hence, due to mathematical induction, $\bm{x}_0  > 0$ guarantees that $\bm{x}(t)  > 0, \; \forall t \in [0,T]$. 
\end{proof}

Now, we derive a closed-form expression for the index function.

\begin{proposition}
\label{pro:routingpolicy}

The index function for Problem \eqref{eq:routing} is
\begin{equation}
\label{eq:arcgammait}
\gamma_i(t) = R - {C_i} \frac{\lambda}{\mu_i} \big[1- e^{-\mu_i (T-t)}\big] , \quad i \in [n], t \in [0, T].
\end{equation}
\end{proposition}

\begin{proof} 
First, note that the costate $y_i(t)$ satisfies the following ODE in this model:
\[
\dot{y}_i(t) = C_i + \mu_i y_i(t), \quad t \in [0, T].
\]
The index function derived from Lemma \ref{lemma:pmp} is $\gamma_i(t) = R + \lambda y_i(t)$.

For $t \in [t_s, t_{s+1}),$ 
\begin{equation*}
\label{eq:affinedotxysolarc}
\begin{split}
x_i(t) & =  x_i(t_s)+ [\frac{\lambda u_{s,i} }{\mu_i} - x_i(t_s)] \, \big[1 - e^{-\mu_i 
 (t-t_s)}\big] \\
y_i(t) & = y_i(t_s) + \big[- \frac{C_i}{\mu_i} - y_i(t_s)\big] \big[1 - e^{\mu_i 
 (t-t_s)}\big].
\end{split}
\end{equation*}
Applying the boundary condition $y_i(T) = 0$, $y_i(t)$ is given by 
\begin{equation}
\label{eq:yitsolarc}
y_i(t) = - \frac{C_i}{\mu_i} \big[1- e^{-\mu_i (T-t)}\big], \quad t \in [0, T].
\end{equation}

Therefore, substituting (\ref{eq:yitsolarc}) into the index function $\gamma_i(t)$, we obtain the closed-form index expression. 
\end{proof}

By Proposition \ref{pro:routingpolicy}, solving Problem \eqref{eq:routing} is significantly simplified. At each time $t$, we compute the index functions for all projects and allocate efforts accordingly.

\subsubsection{Learning a State Feedback Policy using OCT-H}
\label{s:octhex}

In this section, we present a state feedback policy for Problem \eqref{eq:routing} learned by OCT-H. We consider the problem with $ m = 1$ and $n = 2$. For this problem, at each time $t$, we only need to compare $\gamma_1(t), \gamma_2(t)$ and 0 to determine the extremal control. From the derivation above, it is straightforward that the extremal feedback policy depends only on the time $t$ and not on the state variable. Given the parameters $\mu_1 = 0.5, \mu_2 = 1, C_1 = 1, C_2 = 1.5, \lambda = 1, R = 3, T = 10 $, the extremal feedback policy is:
\begin{equation}
\pi(\bm{x}, t) = 
\begin{cases}
\displaystyle (0,1), & \text{if } t < 10 - \log{9} \approx 7.802, \\ 
\displaystyle (1,0), & \text{if } t \geq 10 - \log{9} \approx 7.802.
\end{cases}
\label{eq:queuopt}
\end{equation}

\noindent To generate a training data, we sample 1000 initial states uniformly from the interval $(0,10)^2$ and solve each instance. For each solved instance, we extract 10 feature vectors along with their associated control vectors from each subinterval with constant control. The policy learned by OCT-H is given in Figure \ref{fig:routing}. We observe that this policy is almost identical to the extremal policy in \eqref{eq:queuopt}, with only slight numerical differences.

\begin{figure}
    \centering
    \includegraphics[scale = 0.5]{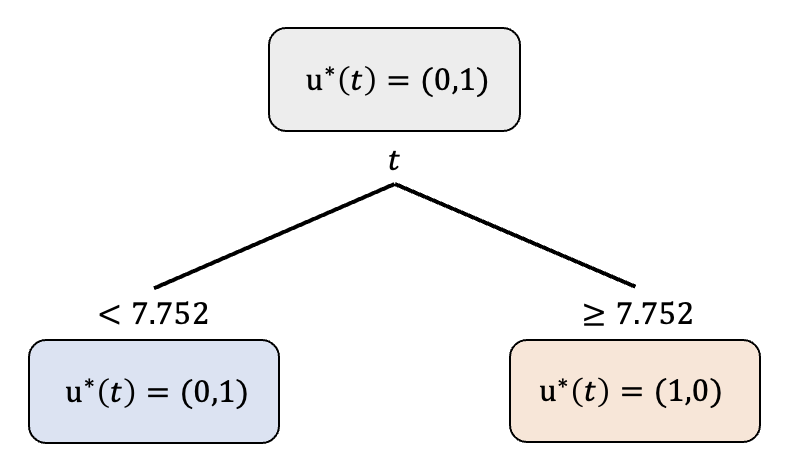}
    \caption{The decision tree OCT-H learned for the infinite server routing problem with $n = 2$.}
    \label{fig:routing}
\end{figure}

\section{Computational Experiments}
\label{sec:exp}

In this section, we present the results of computational experiments. We conduct experiments on three distinct problems with varying sizes and time horizons. The first problem, the machine maintenance problem, has affine dynamics. The other two problems, the epidemic control and the fisheries control problems, have quadratic dynamics. We evaluate the quality of the policy learned using Algorithm \ref{alg:main}, and also assess the relative speed-up compared to Algorithm \ref{alg:shooting}. For all problems considered, the state trajectories automatically satisfy the state constraints $x_i(t) \in (0,H_i), \forall i \in [n],$ regardless of the control trajectories. The proof is provided in Appendix \ref{sec:appendix_proof}.

\textcolor{black}{We note that in all instances considered, the chattering phenomenon did not occur; Algorithm 1 terminated successfully and produced extremal trajectories with a finite number of pieces. Further, when running Algorithm \ref{alg:shooting}, we used three random initializations of $\bm{y}_{0,0}$ for each instance. In all cases, across different initializations for a single instance, the resulting trajectories were either identical or the algorithm simply failed to converge. In other words, we found only a single trajectory satisfying the PMP conditions for each instance. In general, a solution satisfying the PMP conditions need not be globally or even locally optimal, and a optimum may not even exist. However, if an optimum does exist, the solution we found is a strong candidate, given that no other trajectory satisfying the PMP conditions was identified despite multiple random initializations.}

\subsection{Problem Description}
\label{subsec:desc}
\subsubsection{Machine Maintenance}
We consider the classic machine maintenance problem studied in \citep{ks71}. In this problem, for each machine $i \in [n]$, the state variable $x_i(t)$ represents the cumulative probability that machine $i$ has failed by time $t$, while $u_i(t)$ represents the preventive maintenance rate for machine $i$ at time $t$. For each machine \( i \), we denote its natural failure rate, maintenance cost, junk value, and revenue as \( h_i, C_i, L_i, R_i \), respectively. {\color{black} The {failure rate} refers to the instantaneous probability of a machine failing at a given time, provided it is still operational. The natural failure rate \( h_i \) is the baseline failure rate in the absence of maintenance. The junk value \( L_i \) represents the salvage value of the machine once it has failed, and the revenue \( R_i \) indicates the reward obtained from operating machine \( i \).} The primary objective is to maximize the total profit generated by the machines, adjusted for any junk values if a machine fails prematurely.
The objective is to maximize 
$$\int_0^T \sum_{i=1}^n \Big[R_i - C_ih_iu_i(t) + L_i(h(1 - u_i(t))(1 - x_i(t)))\Big] \, dt$$ 
subject to the state equation 
$$\dot{x}_i(t) = h(1 - u_i(t))(1 - x_i(t)).$$ 
To align this problem with the formulation in Section \ref{sec:affine}, we can set the parameters as $\alpha_i^{1} =  \beta_i^{1} = 0$, and $\alpha_i^{0} = h_i,  \beta_i^{0} = - h_i$. In this model, the state trajectory stays in $(0,1)^{n}$ regardless of the control policy.

To generate a problem, we sample the parameters $\bm{h}, \bm{C}, \bm{L}, \bm{R}$ uniformly at random from the intervals $[0, 0.5]^n, [1,3]^n, [2,4]^n, [2,4]^n,$ respectively.

\subsubsection{Epidemic Control}

This problem is based on the SIS epidemic model studied in \citep{pollett24}. It has been shown that the fraction of infected individuals in a stochastic version of the SIS epidemic model, following a continuous-time Markov chain, converges in probability to the solution of a continuous-time deterministic differential equation. The epidemic control problem we consider is derived from this continuous-time deterministic differential equation.

In this problem, the state variable $x_i(t)$ represents the fraction of infected people in subpopulation $i$ and $u_i(t)$ represents the intervention effort for subpopulation $i$ at time $t$. Let $C_i$ be the unit social cost per fraction of infected population, and $P_i$ be the unit intervention cost. The transmission rates $\lambda_i^{1}$ and $\lambda_i^{0}$ correspond to active and inactive intervention cases in subpopulation $i$, respectively. Similarly, $\mu_i^{1}$ and $\mu_i^{0}$ denote the recovery rates under active and inactive intervention, respectively. The objective is to minimize the total cost 
$$\int_0^T \sum_{i=1}^n [C_i(t)x_i(t) + P_iu_i(t)] \, dt$$ 
subject to the state equation 
$$
\dot{x}_i(t) = u_i(t)\bigg[\lambda_i^{1} x_i(t)\bigg(1 - \frac{\mu_i^{1}}{\lambda_i^{1}}- x_i(t)\bigg)\bigg]\ + (1- u_i(t))\bigg[\lambda_i^{0} x_i(t)\bigg(1 - \frac{\mu_i^{0}}{\lambda_i^{0}}- x_i(t)\bigg)\bigg]
.$$ 
In this model, the state trajectory always stays in $(0,1)^{n}$.

To generate a problem, we make two assumptions. First, we assume that in each subpopulation $i$, the intervention cost is always lower than the cost of infection. Second, $\lambda_i^{1} < \mu_i^{1}$ and $\lambda_i^{0} > \mu_i^{0}$, to reflect the effects of active intervention. To implement these assumptions, we first sample $\bm{C}$ uniformly at random from the interval $[0, 1]^n$. We then multiply $\bm{C}$ element-wise  with another vector sampled uniformly at random from $[0, 1]^n$ to compute $\bm{P}$. Next, we sample $\bm{\lambda}^1$ and $\bm{\mu}^0$ uniformly at random from $[2,4]^{n}$. Finally, we randomly sample two vectors uniformly from $[0,0.5]^{n}$ and add them to $\bm{\lambda}^1$ and $\bm{\mu}^0$ to obtain $\bm{\mu}^1$ and $\bm{\lambda}^0,$ respectively. {\color{black} These parameter ranges follow the numerical examples provided in \citet{pollett24}, and similar parameter ranges have been widely adopted in other SIS-model studies \citep{britton_time_2010, hiebeler_moment_2006, BALL2002333}.}

\subsubsection{Fisheries Control}

This problem is based on the classic logistic model of population growth \citep{Ba2011}, extended to the optimal control of fisheries studied by \citep{schae57, gordon91}. In this example, $x_i(t)$ and $u_i(t)$ represent the size of population $i$ at time $t$ and the fishing effort on population $i$ at time $t$, respectively. $r_i$ denotes the intrinsic growth rate of population $i$, and $H_i$ is the maximum sustainable population size. For each population $i$, its catchability coefficient, the unit price of landed fish, and the unit cost of effort are denoted $q_i$, $p_i$, and $C_i$, respectively. The objective is to maximize 
$$\int_0^T \sum_{i=1}^n (p_i q_i x_i(t) - C_i) u_i(t) \, dt$$
subject to the state equation
\[
\dot{x}_i(t) = r_i \left( 1 - \frac{x_i(t)}{H_i} \right) x_i(t) - q_i x_i(t) u_i(t).
\]
To align this problem with the formulation in Section \ref{sec:quadratic}, we can set the parameters as $\beta_i^{0} = \beta_i^{1} = -\frac{r_i}{H_i}$, $\alpha_i^{0} = r_i$, and $\alpha_i^{1} = r_i - q_i$. In this model, the state trajectory always stays in $\prod_{i=1}^n (0,H_i)$.

To generate a problem, we sample the parameters $\bm{r}, \bm{H},\bm{q}, \bm{p}, \bm{C}$ uniformly at random from the intervals $[0, 0.15]^n, [1,6]^n, [0,0.15]^n, [0,2]^n, [0,0.1]^n$, respectively.

\subsection{Experimental Setup}
\label{subsec:exp_set}

\subsubsection{Evaluation of Computational Efficiency and Policy Quality}
\label{subsec:main_exp}
We randomly generate problem instances with varying initial states and solve them using Algorithm \ref{alg:shooting} to generate training data. We sample 3000 initial states uniformly at random from $\prod_{i=1}^n (0,H_i)$, where $H_i$ represents the bounds specific to each problem. For each solved instance, we extract 10 feature vectors along with their associated control vectors from each subinterval with constant control.  For Algorithm \ref{alg:shooting}, we fix $m = \floor*{0.3n}$, $\epsilon = 0.00001$, and $\delta = 0.0001$. 

We evaluate our approach in two different ways. First, we measure simply how well the learned policy imitates the extremal trajectory generated by Algorithm \ref{alg:shooting}. We evaluate the out-of-sample test accuracy on 1000 data points consisting of state-time and control pairs that the decision tree has not seen during training.

Second, to evaluate the ultimate quality of the learned policies, we apply them to problems with varying initial states and compute the associated objective cost. We compare this cost with the extremal objective cost acquired by Algorithm \ref{alg:shooting}. To compute the objective cost associated with a policy, we discretize the dynamics to approximate the continuous-time trajectory and the integral objective values. We measure the {\color{black}quality} of a policy by subtracting the associated objective cost from the extremal objective cost, and then dividing the result by the absolute value of the associated objective cost. {\color{black}We denote this value as PMP-Gap.} We report the maximum of these values across the test instances to analyze the potential {\color{black}quality} of the learned policies for unseen problems, {\color{black}measured as the gap of attained objective cost from the extremal objective cost.}. The number of test instances is 100.

 To quantify the speed-up compared to the shooting method, we compute the ratio between the time required to solve an instance from scratch using Algorithm \ref{alg:shooting} and the inference time of the trained decision tree. {\color{black}When measuring the runtime of Algorithm \ref{alg:shooting}, we averaged the runtimes over the initializations that successfully converged.} This measurement is averaged over 100 test instances. {\color{black}Since inference times for decision trees are consistently on the order of milliseconds without much variability, we do not report them separately. Instead, we focus on the relative speed-up as it provides a more informative metric.}  We also report the OCT-H training time in minutes and seconds. {\color{black} This reported training time excludes both the feature augmentation stage and the preceding data generation steps. It reflects only the time required to train OCT-H once the training data has been prepared.}

{\color{black}
\subsubsection{Robustness Analysis}
\label{subsec:robust_exp}
In addition to the main experiments described above, we evaluate the robustness of our approach from two distinct perspectives. In the first experiment, conducted on the fisheries control problem with $n = 10$ and $T = 1$, we train a policy using the same parameters as those described in Section \ref{subsec:desc}. After training, we evaluate the policy on test instances whose parameters differ from the training set. Specifically, for each parameter used in training (such as $C_i$ for project $i \in [n]$), we uniformly sample a parameter for the test problem from the range  $\left[(1-\Gamma)C_i,\,(1+\Gamma)C_i\right]$,
where $\Gamma$ denotes the noise level. We consider multiple noise levels and generate 100 random parameter configurations for each noise level (i.e., 100 test instances per noise level). For each noise level, we report the maximum {\color{black}PMP-Gap} and illustrate how {\color{black}PMP-Gap} varies with increasing noise. 

In the second robustness experiment, we re-run the experiments in Section \ref{subsec:main_exp} using system parameters sampled from ranges different from those described in Section \ref{subsec:desc}. This experiment evaluates whether our approach is limited to a specific range of parameters. To avoid redundancy, we report these results in Appendix \ref{sec:appendix_exp}, as they are similar to the results from the main experiments.
}

All experiments in this section were conducted on a MacBook Pro with an Apple M2 Pro chip and 16GB of memory. Software for OCT-H is available at \citet{InterpretableAI}. We tune the maximum depth of the tree by grid searching over the list [5,10,15].

\subsection{Experimental Results}
\label{subsec:exp_result}

In Tables \ref{table:machine}, \ref{table:epidemic}, and \ref{table:fisheries}, we report the experimental results for the machine maintenance, the epidemic control, and the fisheries control problems, respectively. Table \ref{table:data_summary} summarizes the training data for OCT-H, where we report the number of feature variables (after applying Algorithm \ref{alg:augmentation}) and the number of distinct control vectors (i.e., prediction targets). {\color{black} In other words, after applying Algorithm 2 to the original training data, we simply count the number of features and the number of distinct control vectors that appear as prediction targets in the training data.}  Finally, Figure \ref{fig:robustness} presents the results of our robustness experiment. We draw the following conclusions.

\subsubsection*{\textbf{Observations}} 
\begin{itemize}
    \item The out-of-sample classification accuracy is consistently high, never falling below 98\% and frequently achieving 100\%. This indicates that the policy learned using Algorithm \ref{alg:main} imitates the extremal trajectory very well. 
    \item Even when the out-of-sample test accuracy falls below 100\%, the maximum {\color{black}PMP-Gap} remains very low, never exceeding 1.8\%. 
    \item The proposed approach significantly outspeeds Algorithm \ref{alg:shooting}, with speed-ups reaching up to more than 26 million times. We also note that the relative speed-up increases with larger $n$ and $T$. This enhancement is attributed to the increased computational demand of solving problems with larger number of projects and longer time horizons.
    \item The training times are typically short, up to several hours in a personal laptop.  demonstrating the practicality of our approach.
    \item {\color{black}While the maximum {\color{black}PMP-Gap } increases as the noise level grows, it remains modest (below 5\%) when the perturbations are small.
    \item The training times for the fisheries control problem, especially for \(n=10\), are generally higher than other problems. This is likely due to the larger number of feature variables and distinct control vectors present in the training dataset for this problem.} 
\end{itemize}

\begin{table*}\centering
\caption{Experiment results for the machine maintenance problem.}
\ra{1.3}
\begin{tabular}{cccccc}
\toprule
$n$ & $T$ &  Training Time (min:s) & Speed-up & Accuracy & Max Suboptimality   \\ 
\midrule

5& \multirow{2}{*}{1}& 2:23& 47625  & 1.00 & 0.0000  \\

10&  & 4:50 & 93751 & 1.00 & 0.0000  \\

\hline

5&  \multirow{2}{*}{5}& 6:49 & $4.90 \times 10^5$ & 1.00 & 0.0000  \\

10&   &19:27& $3.23 \times 10^6$ & 0.98 & 0.0181  \\

\bottomrule
\end{tabular}
\label{table:machine}
\end{table*}

\begin{table*}\centering
\caption{Experiment results for the epidemic control problem.}
\ra{1.3}
\begin{tabular}{cccccc} 
\toprule
$n$ & $T$ &  Training Time (min:s) & Speed-up & Accuracy & Max Suboptimality   \\ 
\midrule

5& \multirow{2}{*}{1} & 7:33 & $2.10 \times 10^5$ & 0.99 & 0.0011  \\

10&  &12:33& $1.08 \times 10^6$ & 1.00 & 0.0000  \\

\hline
5& \multirow{2}{*}{5} & 7:04 & $8.60 \times 10^5$ & 0.99 & 0.0000  \\

10&  & 18:07 &  $2.65 \times 10^7 $& 1.00 & 0.0000  \\

\bottomrule
\end{tabular}
\label{table:epidemic}
\end{table*}

\begin{table*}\centering
\caption{Experiment results for the fisheries control problem.}
\ra{1.3}
\begin{tabular}{cccccc}
\toprule
$n$ & $T$ &  Training Time (min:s) & Speed-up & Accuracy & Max Suboptimality   \\ 
\midrule

5& \multirow{2}{*}{1}& 34:24 &  90653 & 0.98 & 0.0000  \\

10&  & 51:29 & $1.59 \times 10^5$ & 0.99 & 0.0000  \\

\hline

5&  \multirow{2}{*}{5}& 12:40 &  $3.90 \times 10^5$ & 0.99 & 0.0011  \\

10&  & 143:20 & $2.07 \times 10^6$ & 0.98 & 0.0111   \\

\bottomrule
\end{tabular}
\label{table:fisheries}
\end{table*}

\begin{table*}\centering
\caption{Summary of the training data for OCT-H. The Feature column reports the number of feature variables (after augmentation), and the Controls column reports the number of distinct control vectors.}
\ra{1.3}
\begin{tabular}{ccccc}
\toprule
Problem & $n$ & $T$ & Features & Controls  \\ 
\midrule

\multirow{4}{*}{Machine Maintenance}& 5 & 1 & 13  & 3    \\

& 5 & 5 & 15 & 6  \\

& 10& 1 & 24 & 4  \\

& 10& 5 & 27 & 15   \\

\hline

\multirow{4}{*}{Epidemic Control}& 5 & 1 & 21  & 6   \\

& 5 & 5 & 21 &  14 \\

& 10& 1 & 32 &  2  \\

& 10& 5 & 32 &  2 \\
\hline

\multirow{4}{*}{Fisheries Control}& 5 & 1 & 20  & 11  \\

& 5 & 5 & 19 & 6 \\

& 10& 1 & 37 & 21  \\

& 10& 5 & 37 & 22  \\

\bottomrule
\end{tabular}
\label{table:data_summary}
\end{table*}

\begin{figure}
    \centering
    \includegraphics[scale = 0.35]{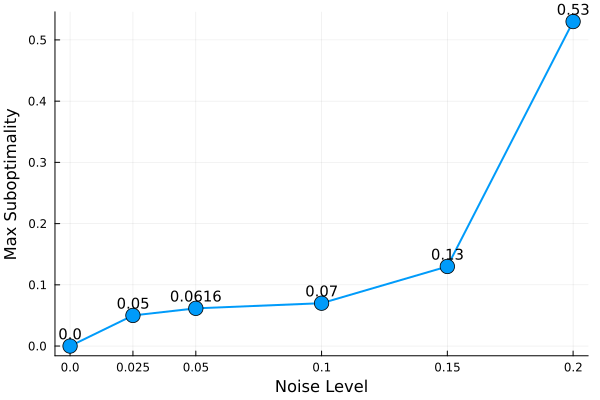}
    \caption{Evaluating robustness on the fisheries control problem. The plot illustrates the maximum {\color{black}PMP-Gap} of the proposed approach as the noise level increases.}
    \label{fig:robustness}
\end{figure}

\FloatBarrier
\section{Conclusions}
\label{sec:conclusion}

We have proposed a machine learning approach to learn a time-dependent state feedback policy for FRMABPs with affine and quadratic state equations. Instead of relying on black-box algorithms that automatically learn nonlinear patterns, we introduce a feature augmentation technique and use OCT-H. Our computational experiments demonstrate that this approach effectively learns high-quality policies for a variety of models across different sizes and time horizons.

Regarding limitations and future extensions, while our experiments demonstrate that the trained policies are reasonably robust to parameter perturbations, the proposed approach is currently designed for fixed parameter configurations. Developing a more flexible and efficient learning algorithm that can adapt to dynamic changes in parameter settings is an important direction for future work.  Additionally, extending our framework to more general fluid MDP models is a promising avenue for further research.

{\color{black}Furthermore, our work faces inherent theoretical limitations: for the models studied, neither the finiteness of the number of subintervals in the PMP extremals nor their global optimality is guaranteed a priori. Nevertheless, our numerical results are encouraging: the controls exhibit finitely many switches, the extremal solutions appear unique, and the computed PMP-Gaps are consistently small and nonnegative. A rigorous theoretical analysis remains necessary to formally establish these properties.}

% Appendix here
% Options are (1) APPENDIX (with or without general title) or
%             (2) APPENDICES (if it has more than one unrelated sections)
% Outcomment the appropriate case if necessary
%
\begin{appendices}

\section{Proofs}\label{sec:appendix_proof}

We prove that for the problems considered in Section \ref{subsec:desc}, the state constraints $0 < x_i(t) < H_i, \forall i \in [n],$ are automatically satisfied. Similar to the logic in the proof of Proposition \ref{pro:routingpositive}, we prove that in each interval $[t_s,t_{s+1})$ with a constant control $\bm{u}_s$, if $x_i(t_s) \in (0,H_i)$, then $x_i(t) \in (0,H_i), \forall t \in [t_s,t_{s+1})$. This implies that if $\bm{x}(0) \in (0,H_i)$, then $\bm{x}(t) \in (0,H_i), \forall t \in [0,T]$.

\begin{proposition}
\label{pro:mainpositive}

For the machine maintenance problem described in Section \ref{subsec:desc}, the state trajectory stays within $(0,1)^{n}$ regardless of the control trajectory.

\end{proposition}

\begin{proof} 
We directly apply the results in \eqref{eq:affineex}. For the machine maintenance problem, $\beta_i(u_{s,i}) = - \alpha_i(u_{s,i})$. If $\beta_i(u_{s,i}) = 0$, then $x_i(t) = x_i(t_s) \in (0,1), \forall t \in [t_s,t_{s+1})$.  If $\beta_i(u_{s,i}) \neq 0$, then
\begin{equation*}
\begin{split}
{x}_i(t)  & = x_i(t_s)+ (1 - x_i(t_s))(1 - e^{-h_i(t-t_s)})\\
& = 1 - e^{-h_i(t-t_s)}(1 - x(t_s)).
\end{split}
\end{equation*}

The first equality proves that $x_i(t) > 0,$ and the second equality proves that $x_i(t) < 1, \forall t \in [t_s, t_{s+1})$. 
\end{proof}

\begin{proposition}
\label{pro:epidemicpositive}

For the epidemic control problem described in Section \ref{subsec:desc}, the state trajectory stays within $(0,1)^{n}$ regardless of the control trajectory.

\end{proposition}

\begin{proof} 

We directly apply the results in \eqref{eq:quadex} that
$$
x_i(t) = \frac{K\alpha_i(u_{s,i})e^{\alpha_i(u_{s,i})t}}{1 - K\beta_i(u_{s,i})e^{\alpha_i(u_{s,i})t}}.
$$

From the above expression, $x_i(t)$ is a monotone function of $t$ in this interval. If $\alpha_i(u_{s,i}) > 0$, then 
\begin{equation*}
\begin{split}
\lim_{t\to\infty}x_i(t)  & = -\frac{\alpha_i(u_{s,i})}{\beta_i(u_{s,i})}\\
& = 1 - \frac{u_{s,i}(\mu_0 - \mu_1) - \mu_0}{u_{s,i}(\lambda_0 - \lambda_1) - \lambda_0} \\
\end{split}
\end{equation*}
Due to the assumption that $\alpha_i(u_{s,i}) > 0$,
$$
u_{s,i}(\mu_0 - \mu_1) - \mu_0 > u_{s,i}(\lambda_0 - \lambda_1) - \lambda_0,
$$
where $u_{s,i}(\mu_0 - \mu_1) - \mu_0 = \mu_0(u_{s,i} - 1) - u_{s,i}\mu_1 < 0$. Hence, $-\frac{\alpha_i(u_{s,i})}{\beta_i(u_{s,i})} \in (0,1)$.

If $\alpha_i(u_{s,i}) < 0$, $\lim_{t\to\infty}x_i(t) = 0 $. This guarantees that $x_i(t) \in (0,1)$ in this interval.
\end{proof}

\begin{proposition}
\label{pro:fisheriespositive}

For the fisheries control problem described in Section \ref{subsec:desc}, the state trajectory stays within $(0,H_i)^{n}$ regardless of the control trajectory.

\end{proposition}

\begin{proof} 

The proof is almost identical to the case of epidemic control. Again, we directly apply the results in \eqref{eq:quadex} that
$$
x_i(t) = \frac{K\alpha_i(u_{s,i})e^{\alpha_i(u_{s,i})t}}{1 - K\beta_i(u_{s,i})e^{\alpha_i(u_{s,i})t}}.
$$

From the above expression, $x_i(t)$ is a monotone function of $t$ in this interval. If $\alpha_i(u_{s,i}) > 0$, then 
\begin{equation*}
\begin{split}
\lim_{t\to\infty}x_i(t)  & = -\frac{\alpha_i(u_{s,i})}{\beta_i(u_{s,i})}\\
& = \frac{H_i(r_i - q_iu_{s,i})}{r_i} \\
\end{split}
\end{equation*}
Due to the assumption that $\alpha_i(u_{s,i}) =  r_i - q_iu_{s,i} > 0$,
$\frac{H_i(r_i - q_iu_{s,i})}{r_i} \in (0,H_i)$.

If $\alpha_i(u_{s,i}) < 0$, $\lim_{t\to\infty}x_i(t) = 0 $. This guarantees that $x_i(t) \in (0,H_i)$ in this interval.
\end{proof}

\section{Additional Experiments}\label{sec:appendix_exp}

\subsection{Robustness Experiment}
\label{app:robustness}

In this appendix, we present the results of the second robustness experiment described in Section \ref{subsec:robust_exp}, conducted on the machine maintenance problem with \(n = 10\) and \(T = 1\). In the main experiment, the parameters \(\bm{h}\), \(\bm{C}\), \(\bm{L}\), and \(\bm{R}\) were uniformly sampled from the intervals \([0,0.5]^n\), \([1,3]^n\), \([2,4]^n\), and \([2,4]^n\), respectively. Here, we run three additional experiments using alternative parameter ranges: \(\bm{L}\) is sampled from \([0,1]\), \([4,5]\), and \([5,6]\), while \(\bm{R}\) is sampled from \([4,5]\), \([0,1]\), and \([5,6]\), respectively. 

Table \ref{table:machine_additional} reports these results and labels them as experiments 1, 2, and 3. Table \ref{table:data_summary_appendix} reports the summary of the training data. The outcomes are similar to those in Section \ref{subsec:exp_result}, with perfect training accuracy and zero {\color{black}PMP-Gap}. The relative speed-up is higher than the results reported in Table \ref{table:machine}, indicating that for the parameter ranges used in this experiment, Algorithm \ref{alg:shooting} requires a longer computation time. The training time for the Experiment 2 is noticeably higher than those for the other two experiments, likely due to the larger number of control vectors in the training data.

\begin{table*}\centering
\caption{Experiment results for the machine maintenance problem with different paremeter ranges.}
\ra{1.3}
\begin{tabular}{ccccc}
\toprule
Experiment &  Training Time (min:s) & Speed-up & Accuracy & Max Suboptimality   \\ 
\midrule

1 & 02:19 & $1.90 \times 10^5$  & 1.00 & 0.0000  \\

2  & 38:44 & $3.98 \times 10^5$ & 1.00 & 0.0000  \\

3  & 08:36 & $2.44 \times 10^5$ & 1.00 & 0.0000  \\

\bottomrule
\end{tabular}
\label{table:machine_additional}
\end{table*}

\begin{table*}\centering
\caption{Summary of the training data for OCT-H for the additional experiments using different parameter ranges. }
\ra{1.3}
\begin{tabular}{ccc}
\toprule
Experiment& Features & Controls   \\ 
\midrule

1& 22 & 2    \\

2& 23 & 15    \\

3& 22 & 5  \\

\bottomrule
\end{tabular}
\label{table:data_summary_appendix}
\end{table*}

{\color{black}

\subsection{Runtime of Algorithm \ref{alg:shooting}}
\label{app:shooting}

We evaluate the runtime of Algorithm \ref{alg:shooting} across different values of $n$ and $T$. Table \ref{table:runtime_summary} reports the average runtimes on the test instances from the main experiments in Section \ref{sec:exp}. Table \ref{table:runtimes} fixes $T=5$ and varies $n$ to further analyze runtime. For values of $n$ not included in the main experiments, we run Algorithm \ref{alg:shooting} on 100 instances and report the average runtime. The results show that as $n$ and $T$ increase, the runtime quickly becomes too large for real-time computation of optimal controls across states. In particular, both larger horizons $T$ and larger state dimensions $n$ significantly increase the computational burden.

\begin{table*}\centering
\caption{Average runtime (seconds) of Algorithm \ref{alg:shooting}.}
\ra{1.3}
\begin{tabular}{cccc}
\toprule
Problem & $n$ & $T$ & Runtime (seconds) \\ 
\midrule

Machine Maintenance & 5  & 1 & 4.86   \\
                    & 5  & 5 & 52.31  \\
                    & 10 & 1 & 9.84   \\
                    & 10 & 5 & 151.56 \\

\hline

Epidemic Control    & 5  & 1 & 28.33  \\
                    & 5  & 5 & 233.82 \\
                    & 10 & 1 & 125.93 \\
                    & 10 & 5 & 1547.49 \\

\hline

Fisheries Control   & 5  & 1 & 6.55   \\
                    & 5  & 5 & 175.97 \\
                    & 10 & 1 & 8.85   \\
                    & 10 & 5 & 409.75 \\

\bottomrule
\end{tabular}
\label{table:runtime_summary}
\end{table*}

\begin{table*}\centering
\caption{Runtime (seconds) for different problem types across different values of $n$ with $T = 5$.}
\ra{1.3}
\begin{tabular}{cccc}
\toprule
$n$ & Machine Maintenance & Epidemic Control & Fisheries Control \\
\midrule
4  & 38.81  & 239.27  & 56.25  \\
5  & 52.31  & 233.83  & 175.97 \\
7  & 93.68  & 621.90  & 190.68 \\
10 & 151.56 & 1547.49 & 409.75 \\
12 & 220.90 & 2880.41 & 468.13 \\
\bottomrule
\end{tabular}
\label{table:runtimes}
\end{table*}

}

\end{appendices}

\FloatBarrier
\section*{Statements and Declarations}

\begin{itemize}
\item \textbf{Funding:} Jos\'e Ni\~no-Mora’s work was supported in part by Universidad Carlos III de Madrid (UC3M) through an internal research program grant and a grant for the acquisition of research tools. He also gratefully acknowledges Dimitris Bertsimas for support of two summer research visits to MIT’s Operations Research Center, during which part of this work was carried out.
\item \textbf{Conflict of interest:} The authors have no competing interests to declare that are relevant to the content of this article.
\end{itemize}
%
%   or
%
% \begin{APPENDICES}
% \section{<Title of Section A>}
% \section{<Title of Section B>}
% etc
% \end{APPENDICES}

%%
%\theendnotes

% Acknowledgments here
%\ACKNOWLEDGMENT{}

% References here (outcomment the appropriate case)

% 
\FloatBarrier
\bibliography{reference} % if more than one, comma separated

% CASE 2: BiBTeX used to generate mypaper.bbl (to be further fine tuned)
%\input{mypaper.bbl} % outcomment this line in Case 2

%If you don't use BiBTex, you can manually itemize references as shown below.

%%%%%%%%%%%%%%%%%
\end{document}